\documentclass[11pt,twoside]{article}
\usepackage{jair}
\usepackage{theapa}
\usepackage{amssymb}
\input epsf 

\jairheading{26}{2006}{127-151}{8/05}{6/06}
\ShortHeadings{Admissible and Restrained Revision}
{Booth \& Meyer}
\firstpageno{127}

\newtheorem{theorem}{\bf Theorem}
\newtheorem{lemma}[theorem]{\bf Lemma}
\newtheorem{proposition}{\bf Proposition}

\newtheorem{definition}{\bf Definition}
\newtheorem{example}{\bf Example}
\newcommand{\qed}{$\Box$}

\newenvironment{proof}{\noindent {\bf Proof:}}{\hfill \qed\\}

\newcommand{\E}{{\mathbb E}}
\newcommand{\ER}{{\mathbb ER}}
\newcommand{\F}{{\mathbb F}}
\newcommand{\counteracts}{\leftrightsquigarrow}
\newcommand{\ncounteracts}{\not\leftrightsquigarrow}
\newcommand{\tl}{\triangleleft}

\newcommand{\R}{\mathrm{R}}

\begin{document}

\title{Admissible and Restrained Revision}
\author{\name Richard Booth \email richard.b@msu.ac.th \\
\addr Faculty of Informatics \\
Mahasarakham University \\
Mahasarakham 44150, Thailand 
\AND
\name Thomas Meyer \email  Thomas.Meyer@nicta.com.au\\
\addr National ICT Australia and \\University of New South Wales\\
223 Anzac Parade\\
Kensington, NSW 2052, Australia
}

\maketitle

\begin{abstract}
As partial justification of their framework for iterated belief
revision Darwiche and Pearl convincingly argued against Boutilier's
natural revision and provided a prototypical revision operator that
fits into their scheme.  We show that the Darwiche-Pearl arguments
lead naturally to the acceptance of a smaller class of operators
which we refer to as {\em admissible}. Admissible
revision ensures that the penultimate input is not ignored
completely, thereby eliminating natural revision, but includes the
Darwiche-Pearl operator, Nayak's lexicographic revision operator,
and a newly introduced operator called \emph{restrained revision}.
We demonstrate that restrained revision is the most conservative of admissible
revision operators,  effecting as few changes as possible, 
while lexicographic revision is the least
conservative, and point out that restrained
revision can also be viewed as a composite operator, consisting of
natural revision preceded by an application of a ``backwards
revision'' operator previously studied by Papini. Finally, we propose
the establishment of a principled approach for choosing an appropriate
revision operator in different contexts and discuss future work. 
\end{abstract}

\section{Introduction}
The ability to rationally change one's knowledge base in the face of new information
which possibly contradicts the currently held beliefs is a basic characteristic
of intelligent behaviour. Thus the question of {\em belief revision} is of
crucial importance in Artificial Intelligence. In the last twenty years this 
question has received considerable attention, starting from the work of
Alchourr{\'o}n, G{\"ardenfors}, and Makinson \citeyear{Alchourron-ea:85a} -- usually abbreviated to just {\em AGM} -- who proposed a set of rationality 
postulates which 
any reasonable revision operator should satisfy. A
semantic
construction of revision operators was later provided by Katsuno and Mendelzon \citeyear{Katsuno-ea:91a}, 
according to which an agent has in its mind a plausibility ordering
-- a total preorder -- over the set of possible worlds, with the knowledge base associated 
to this ordering being identified with the set of sentences true in all the most plausible
worlds. This approach dates back to the work of Lewis \citeyear{Lewis:73a} on counterfactuals. It was
introduced into the belief revision literature by Grove \citeyear{Grove:88a} and Spohn \citeyear{Spohn:88a}. 
Given a new sentence -- or
{\em epistemic input} -- $\alpha$, the revised knowledge base is set as the
set of sentences true in all the most plausible worlds in which $\alpha$ holds.
As was shown by  Katsuno and Mendelzon \citeyear{Katsuno-ea:91a}, the family of operators defined by
this construction coincides exactly with the family of operators satisfying
the AGM postulates. Due to its intuitive appeal, this construction came to be
widely used in the area. However, researchers soon began to notice a
deficiency with it -- although it prescribes how to
obtain a new {\em knowledge
base}, it remains silent on how to obtain a new {\em plausibility ordering} 
which can then serve as
a target for the next epistemic input. Thus it is not rich
enough to deal adequately with the problem of {\em iterated} belief revision. 
This paper is a contribution to the study of this problem. 

Most iterated revision schemes are sensitive to the history of belief changes\footnote{%
An {\em external} revision scheme like
that of Areces and Becher \citeyear{Areces-ea:2001a} and Freund and Lehmann \citeyear{Freund-ea:94a} is not.},
based on a version of the ``most recent is best'' argument, where the newest information is of higher priority than anything else in
the knowledge base. Arguably the most extreme case of this is Nayak's lexicographic revision
\cite{Nayak:94b,Nayak-ea:2003a}. However, there are operators where,
once {\em admitted} to the knowledge base, it rapidly becomes as much of a candidate for removal as anything else in the set when another, newer, piece of information comes along,
Boutilier's natural revision \citeyear{Boutilier:93a,Boutilier:96a} being a case in point.
A dual to this is what  Rott \citeyear{Rott:2003a} terms {\em radical} revision where the new
information is accepted with maximal, irremediable entrenchment -- see also Segerberg 
\citeyear{Segerberg:98a}.
Another issue to consider is the problem termed
{\em temporal incoherence} \cite{Rott:2003a}: \begin{quote}the comparative
recency of information should translate systematically into
comparative importance, strength or entrenchment\end{quote}

In an influential paper Darwiche and Pearl \citeyear{Darwiche-ea:97a}
proposed a framework for iterated revision. Their proposal is
characterised in terms of sets of syntactic and semantic postulates,
but can also be viewed from the perspective of conditional beliefs.
It is an extension of the formulation by Katsuno and Mendelzon  \citeyear{Katsuno-ea:91a} of 
AGM revision \cite{Alchourron-ea:85a}. 
To justify their proposal Darwiche and Pearl mount a comprehensive
argument. The argument includes a critique of natural revision,
which is shown to admit too few changes. In addition, they provide a
concrete revision operator which is shown to satisfy their
postulates. In many ways this can be seen as the prototypical
Darwiche-Pearl operator.  It is instructive to observe that the two
best-known operators satisfying the Darwiche-Pearl postulates,
natural revision and lexicographic revision, form the opposite
extremes of the Darwiche-Pearl framework: Natural revision is the
most conservative Darwiche-Pearl operator, in the sense that it 
effects as few changes as possible, while lexicographic
revision is the least conservative.

In this paper we show that the Darwiche-Pearl arguments lead
naturally to the acceptance of a smaller class of operators which we
refer to as {\em admissible}. We provide characterisations of
admissible revision, in terms of syntactic as well as semantic
postulates. Admissible revision ensures that the penultimate input
is not ignored completely. A consequence of this is that natural
revision is eliminated. On the other hand, admissible revision
includes the prototypical Darwiche-Pearl operator as well as
lexicographic revision, the latter result also showing that
lexicographic revision is the least conservative of the admissible
operators. The removal of natural revision from the scene leaves a
gap which is filled by the introduction of a new operator we refer
to as \emph{restrained revision}.  It is the most conservative of
admissible revision operators, and can thus be seen as an
appropriate replacement of natural revision.  We give a syntactic
and a semantic characterisation of restrained revision, and
demonstrate that it satisfies desirable properties. In particular,
and unlike lexicographic revision, it ensures that older information
is not discarded unnecessarily, and it shows that the problem of
temporal incoherence can be dealt with.

Although natural revision does not feature in the class of
admissible revision operators, we show that it still has a role to
play in iterated revision, provided it is first tempered
appropriately. We show that restrained revision can also be viewed
as a composite operator, consisting of natural revision preceded by
an application of a ``backwards revision'' operator previously
studied by Papini \citeyear{Papini:98b}.

The paper is organised as follows. After outlining some notation, we review the Darwiche-Pearl
framework in Section \ref{Sec:DP}. This is followed by a discussion of admissible revision in
Section \ref{Sec:admissible}. In Section \ref{Sec:restrained} we introduce restrained revision,
and in Section \ref{Sec:Composite} we show how it can be defined as a composite operator. 
Section \ref{Sec:Choice} discusses the possibility of enriching epistemic states as a way of
determining the appropriate admissible revision operator in a particular context. In this
section we also conclude and briefly discuss some future work.

\subsection{Notation}
We assume a finitely generated propositional language $L$ which includes the constants
$\top$ and $\bot$, is closed under the
usual propositional connectives, and is equipped with a classical
model-theoretic semantics.
$V$ is
the set of valuations of $L$ and $[\alpha]$ (or $[B]$) is the set of models of
$\alpha\in L$ (or $B\subseteq L)$. Classical entailment is denoted by $\vDash$ and logical equivalence
by $\equiv$. We also use $Cn$ to denote the operation of closure under classical entailment. 
Greek letters $\alpha, \beta, \ldots$ stand for arbitrary sentences. In our examples we 
sometimes use the lower case letters $p$, $q$, and $r$ as propositional atoms, and  sequences of 
0s and 1s to denote the valuations of the language. For example, $01$ denotes the valuation,
in a language generated by $p$ and $q$, in which $p$ is assigned the value 0 and $q$ the
value 1, while  $011$ denotes the valuation,
in a language generated by $p$, $q$ and $r$, in which $p$ is assigned the value 0 and both $q$ and $r$ the value 1. Whenever we use the term {\em knowledge base} we will always mean a set of 
sentences $X$ which is {\em deductively closed}, i.e., $X = Cn(X)$. 

\section{Darwiche-Pearl Revision}
\label{Sec:DP}%
Darwiche and Pearl \citeyear{Darwiche-ea:97a} reformulated the AGM
postulates \cite{Alchourron-ea:85a} to be compatible with their
suggested approach to iterated revision. This necessitated a move
from knowledge bases to \emph{epistemic states}. An epistemic state
contains, in addition to a knowledge base, all the information
needed for coherent reasoning including, in particular, the strategy
for belief revision which the agent wishes to employ at a given
time. Darwiche and Pearl consider epistemic states as abstract entities, and do not
provide a single formal representation. It is thus possible to talk about two epistemic states 
$\E$ and $\F$ being identical (denoted by $\E=\F$) , but yet syntactically different.\footnote{%
Personal communication with Adnan Darwiche.}
This has to be borne in mind below, particularly when considering postulate ($\E\ast 5)$.  
In Darwiche and Pearl's reformulated postulates $\ast$ is a belief change
operator on epistemic states, not knowledge bases.
We denote by $B(\E$) the knowledge base extracted from an epistemic state $\E$.
\begin{description}
\item [($\E\ast$1)] $B(\E\ast\alpha) =Cn(B(\E\ast\alpha))$
\item [($\E\ast$2)] $\alpha \in B(\E\ast\alpha)$
\item [($\E\ast$3)] $B(\E\ast\alpha)\subseteq B(\E)+\alpha $
\item [($\E\ast$4)] If $\lnot\alpha \notin B(\E)$ then $B(\E)+\alpha\subseteq B(\E\ast\alpha)$
\item [($\E\ast$5)] If $\E=\F$ and $\alpha \equiv \beta $ then $B(\E\ast\alpha) =B(\F\ast\beta)$
\item [($\E\ast$6)] $\bot \in B(\E\ast\alpha)$ iff $\vDash \lnot \alpha$
\item [($\E\ast$7)] $B(\E\ast(\alpha \wedge \beta ))\subseteq B(\E\ast\alpha)+\beta $
\item [($\E\ast$8)] If $\lnot \beta \notin B(\E\ast\alpha)$ then $B(\E\ast\alpha)+\beta\subseteq B(\E\ast(\alpha \wedge \beta))$
\end{description}
Darwiche and Pearl then show, via a representation result similar to that of Katsuno and Mendelzon
\citeyear{Katsuno-ea:91a},
that revision on epistemic states can be represented in terms of plausibility orderings associated 
with epistemic states.\footnote{%
Alternative frameworks for studying iterated revision, both based on using sequences of sentences rather than plausibility orderings, are those of Lehmann \citeyear{Lehmann} and 
Konieczny and Pino-P{\'e}rez \citeyear{KP}.}  
More specifically, every epistemic state $\E$ has associated with it a total 
preorder $\preceq_{\E}$ on all valuations,  with elements lower down in the ordering deemed more
plausible. Moreover, 
for any two epistemic states $\E$ and $\F$ which are identical (but may be syntactically different), 
it has to be the case that $\preceq_{\E}=\preceq_{\F}$. Let 
$\min(\alpha,\preceq_{\E})$ denote the minimal models of $\alpha$
under $\preceq_{\E}$. The knowledge base associated with the epistemic
state is obtained by considering the minimal models in
$\preceq_{\E}$ i.e., $[B(\E)] =\min(\top,\preceq_{\E})$. Observe that this means that 
$B(\E)$ has to be consistent.
This requirement enables us to obtain a unique knowledge base from the total preorder 
$\preceq_{\E}$. Preservation of the results in this paper when this requirement is relaxed is possible, but technically messy. 

The observant reader will note that our assumption of a consistent $B(\E)$ is incompatible with
a successful revision by $\bot$. This requires that we jettison ($\E\ast$6) and insist on consistent epistemic inputs only. (The left-to-right direction of ($\E\ast$6) 
 is rendered superfluous by ($\E\ast$1) and the
assumption that knowledge bases extracted from all epistemic states
have to be consistent.)
The other difference between the original AGM postulates and the Darwiche-Pearl reformulation
-- first inspired by a critical observation by Freund and Lehmann \citeyear{Freund-ea:94a} -- occurs in ($\E\ast$5), which  states that revising
by logically equivalent sentences results in epistemic states with identical associated knowledge bases.
This is a weakening of the original AGM postulate, phrased in our notation as follows:
\begin{description}
\item [($B\ast$5)] If $B(\E) = B(\F)$ and $\alpha \equiv \beta $ then $B(\E\ast\alpha) =B(\F\ast\beta)$
\end{description}
($B\ast$5) states that two epistemic states with identical associated {\em knowledge bases} will, after 
having been revised by equivalent inputs, produce two epistemic states with identical
associated knowledge bases. This is stronger than ($\E\ast$5) which requires equivalent 
associated knowledge bases only if the original {\em epistemic states} were identical.
We shall refer to the reformulated AGM
postulates, with ($\E\ast$6) removed,  as DP-AGM.
 
DP-AGM guarantees a unique extracted knowledge base when revision
by $\alpha$ is performed. It sets $[B(\E\ast\alpha)]$ equal to $\min(\alpha,\preceq_{\E})$
and thereby fixes the most plausible valuations in $\preceq_{\E\ast\alpha}$. 
However, it places no 
restriction on the rest of the ordering. The purpose of the Darwiche-Pearl framework
is to constrain this remaining part of the new ordering. It is done by way of a set of 
postulates for iterated revision \cite{Darwiche-ea:97a}.
(Throughout the paper we follow the convention that $\ast$ is left associative.) 
\begin{description}
\item[(C1)] If $\beta \vDash \alpha$ then $B(\E  \ast \alpha \ast
\beta) = B(\E \ast \beta)$

\item [(C2)] If $\beta \vDash \neg\alpha$ then $B(\E \ast \alpha \ast
\beta) = B(\E \ast \beta)$

\item [(C3)] If $\alpha\in B(\E \ast \beta)$ then $\alpha\in B(\E \ast
\alpha \ast \beta)$

\item [(C4)] If $\lnot\alpha\notin B(\E \ast \beta)$  then $\lnot\alpha\notin B(\E
\ast \alpha \ast \beta)$
\end{description}
\noindent 
The postulate (C1) states that when two pieces of
information---one more specific than the other---arrive, the first is made
redundant by the second.  (C2) says that when two contradictory
epistemic inputs arrive, the second one prevails;  the second evidence
alone yields the same knowledge base.  (C3)
says that a piece of evidence $\alpha$ should be retained after
accommodating more recent evidence $\beta$ that entails $\alpha$ given the
current knowledge base. (C4) simply says that no epistemic input can act
as its own defeater. We shall refer to the class of belief revision operators satisfying DP-AGM and (C1) to (C4) as \emph{DP-revision}.
The following are the corresponding semantic versions (with $v, w \in V$):
\begin{description}
\item[(CR1)]  If $v\in [\alpha], w\in [\alpha] $
then $v\preceq _{\E }w$ iff $v\preceq _{\E \ast \alpha }w$
\item[(CR2)]  If $v\in [\lnot \alpha], w\in
[\lnot \alpha] $ then $v\preceq _{\E }w$ iff $v\preceq _{\E \ast
\alpha }w$
\item[(CR3)]  If $v\in [\alpha], w\in [\lnot
\alpha]$ then $v\prec _{\E }w$ only if $v\prec _{\E \ast
\alpha }w$
\item[(CR4)]  If $v\in [\alpha], w\in [\lnot
\alpha] $ then $v\preceq _{\E }w$ only if $v\preceq _{\E \ast
\alpha }w$
\end{description}
(CR1) states that the relative ordering between $\alpha$-worlds remain unchanged following
an $\alpha$-revision, while (CR2) requires the same for $\lnot\alpha$-worlds. (CR3) requires 
that, for  an $\alpha$-world strictly more plausible than a $\lnot\alpha$-world, this relationship be 
retained after an $\alpha$-revision, and (CR4) requires the same for weak plausibility. 
Darwiche and Pearl showed that, given DP-AGM,  a precise correspondence obtains between
(C$i$) and (CR$i$) above ($i = 1,2,3,4$).

One of the guiding principles of belief revision is the principle of
minimal change: changes to a belief state ought to be kept to a
minimum. What is not always clear is what ought to be minimised. In
AGM theory the prevailing wisdom is that minimal change refers to
the sets of sentences corresponding to knowledge bases. But there
are other interpretations. With the move from knowledge bases to
epistemic states, minimal change can be defined in terms of the
fewest possible changes to the associated plausibility ordering
$\preceq_{\E}$. In what follows we will frequently have the
opportunity to refer to the latter interpretation of minimal change.
See also the discussion of this principle by Rott \citeyear{dogmas}.

\section{Admissible Revision}
\label{Sec:admissible}%
In this section we consider two of the best-known
DP-operators, and propose three postulates to be added to the Darwiche-Pearl
framework. The first is more 
of a correction than a strengthening. We show that the Darwiche-Pearl representation of
the principle of the irrelevance of syntax is too weak and suggest 
an appropriate strengthened postulate. The second is suggested by some of the arguments
advanced by Darwiche and Pearl themselves. It eliminates one of the
operators they criticise, and is satisfied by the sole operator they
provide as an instance of their framework. The addition of these two postulates to the 
Darwiche-Pearl framework leads to the definition of the class of \emph{admissible} revision
operators. Finally, we point out a problem with Nayak's well-known lexicographic revision
operator and propose a third postulate to be added. The consequences of 
insisting on the addition of this third postulate are discussed in detail in Section \ref{Sec:restrained}.

As mentioned in Section \ref{Sec:DP}, Darwiche and Pearl replaced the original AGM postulate
($B\ast 5$) with ($\E\ast 5$). Both are attempts at an  appropriate formulation of the principle of  the irrelevance of  syntax, popularised by Dalal \citeyear{Dalal:88a}. But whereas ($B\ast 5$) has been shown to be too 
strong, as shown by Darwiche and Pearl \citeyear{Darwiche-ea:97a}, closer inspection reveals that  ($\E\ast 5$) is too {\em weak}. To be more precise, it fails as an adequate formulation of syntax irrelevance for \emph{iterated} revision. It specifies that revision by two equivalent sentences should produce epistemic states with identical associated knowledge bases, but does not require that these epistemic states, after another revision by two equivalent sentences, also have to produce epistemic states with identical associated knowledge bases.  So, as can be seen from the following example, under DP-AGM (and indeed, even if (C1) 
to (C4) are added) it is possible for $B(E\ast\alpha\ast\gamma)$ to differ from $B(E\ast\beta\ast\delta)$ even if  $\alpha$ is equivalent to 
$\beta$ and $\gamma$ is equivalent to $\delta$. 

\begin{example}
Consider a propositional language generated by the two atoms $p$ and $q$
and let $\E$ be an 
epistemic state such that $B(\E)=Cn(p\vee q)$. Now consider the two epistemic states 
$\E'$ and $\E''$  such that $B(\E')=B(\E'')=Cn(p)$, $01\prec_{\E'}00$ and $00\prec_{\E''}01$.
Observe that this gives complete descriptions of $\preceq_{\E'}$ and $\preceq_{\E''}$. It is tedious, but not difficult, 
to verify that setting $\preceq_{\E\ast p}=\preceq_{\E'}$ and $\preceq_{\E\ast\lnot\lnot p}=\preceq_{\E''}$ is compatible with DP-AGM.
But observe then that $B(\E\ast p\ast\lnot p)=Cn(\lnot p\wedge q)$, while $B(\E\ast\lnot\lnot p\ast\lnot p)=Cn(\lnot p\wedge\lnot q)$.    
\end{example}

\noindent As a consequence of this, we propose that ($\E\ast$5) be 
replaced by the following postulate:
\begin{description}
\item[($\E\ast$5$'$)] If $\E=\F$, $\alpha\equiv\beta$ and $\delta\equiv\gamma$ then 
$B(\E\ast\alpha\ast\gamma)=B(\F\ast\beta\ast\delta)$
\end{description}
The semantic equivalent of ($\E\ast$5$'$) looks like this:
\begin{description}
\item [($\ER\ast$5$'$)] If $\E=\F$ and $\alpha \equiv \beta $ then $\preceq_{\E\ast\alpha} =\preceq_{\F\ast\beta}$
\end{description}
($\ER\ast$5$'$) states that the revision of two identical epistemic states by two equivalent sentences 
has to result in epistemic states with identical associated total preorders, not just in epistemic states with identical 
associated knowledge bases.
\begin{proposition}
($\E\ast$5$'$) and  ($\ER\ast$5$'$) are equivalent, given DP-AGM.  
\end{proposition}
\begin{proof} 
The proof that ($\E\ast$5$'$) follows from ($\ER\ast$5$'$) is straightfoward. For the converse, 
suppose that ($\ER\ast$5$'$) does not hold; i.e. 
$\alpha\equiv\beta$ and $\preceq_{\F\ast\alpha}\neq\preceq_{\E\ast\beta}$
for some $\alpha$ and $\beta$.
This means 
there exist $x,y\in V$ such that $x\preceq_{\E\ast\alpha}y$ but $y\prec_{\F\ast\beta}x$.
Now let $\gamma$ be such that $[\gamma]=\{x,y\}$. Then $x\in[B(\E\ast\alpha\ast\gamma)]$, 
but  $[B(\F\ast\beta\ast\gamma)]=\{y\}$, and so 
$B(\E\ast\alpha\ast\gamma)\neq B(\F\ast\beta\ast\gamma)$; a violation of ($\E\ast$5$'$).
\end{proof}

\noindent From this it should already 
be clear that ($\E\ast$5$'$) is a desirable property. This 
view is bolstered further by observing that all the well-known iterated revision operators satisfy it;
natural revision, the Darwiche-Pearl operator $\bullet$, and Nayak's  lexicographic revision, 
the first and third of which are to be discussed in detail below.
In fact, we conjecture that Darwiche and Pearl's intention was to replace ($B\ast$5) with ($\E\ast$5$'$), not with ($\E\ast$5) and propose this as a permanent replacement.
\begin{definition}
The set of postulates obtained by replacing ($\E\ast$5) with ($\E\ast$5$'$) in DP-AGM is
defined as RAGM.
\end{definition}
Observe that RAGM, like DP-AGM, guarantees that $[B(\E\ast\alpha)]=\min(\alpha,\preceq_\E)$.

Rule ($\E\ast$5$'$) is the first of the new postulates we want to add to the Darwiche-Pearl
framework. We now lead up to the second.
One of the oldest known DP-operators is
{\em natural revision}, usually credited to Boutilier \citeyear{Boutilier:93a,Boutilier:96a}, although
the idea can also be found in \cite{Spohn:88a}.
Its main feature is the application of the principle of minimal
change to epistemic states.  It is characterised by DP-AGM plus the following postulate:
\begin{description}
\item[(CB)] If $\lnot\beta\in B(\E\ast\alpha)$ then $B(\E\ast\alpha\ast\beta) = B(\E\ast\beta)$
\end{description}
(CB) requires that, whenever $B(\E\ast\alpha)$ is inconsistent with $\beta$, revising $\E\ast\alpha$
with $\beta$ will completely ignore the revision by $\alpha$.
Its semantic counterpart is as follows:
\begin{description}
\item[(CBR)] For $v,w\notin [B(\E\ast\alpha)]$, $v\preceq_{\E\ast\alpha} w$ iff $v\preceq_{\E} w$
\end{description}
As shown by Darwiche and Pearl \citeyear{Darwiche-ea:97a}, natural revision minimises changes in \emph{conditional beliefs}, 
with $\beta\mid\alpha$ being a conditional belief of an epistemic state $\E$ iff 
$\beta\in B(\E*\alpha)$. In fact, Darwiche and Pearl show (Lemma 1, p. 7), that keeping 
$\preceq_\E$ and $\preceq_{\E*\alpha}$ as similar as possible has the effect of minimising the changes in 
conditional beliefs to a revision. So, from (CBR) it is clear that natural revision is an application of
minimal change to epistemic states. It requires that,  barring the
changes mandated by DP-AGM, the relative ordering of valuations
remains unchanged, thus keeping $\preceq_{\E*\alpha}$ as similar as possible to 
$\preceq_\E$. In that sense then, natural revision is the most conservative of
\emph{all} DP-operators. Such a strict adherence to minimal change
is inadvisable and needs to be tempered appropriately, an issue that
will be addressed in  Section \ref{Sec:Composite}. Darwiche and
Pearl have shown that (CB) is too strong, and that natural revision
is not all that natural, sometimes yielding counterintuitive
results.
\begin{example}
\label{Ex:red-bird}%
\cite{Darwiche-ea:97a} We encounter a strange animal and it appears
to be a bird, so we believe it is one. As it comes closer, we see
clearly that the animal is red, so we believe it is a red bird. To
remove further doubts we call in a bird expert who examines it and
concludes that it is not a bird, but some sort of animal. Should we
still believe the animal is red? (CB) tells us we should no longer
believe it is red. This can be seen by substituting $B(\E) =
Cn(\lnot\beta)=Cn(\mathtt{bird})$ and $\alpha\equiv \mathtt{red}$ in (CB), instructing
us to totally ignore the observation $\alpha$ as if it had never
taken place.
\end{example}
Given Example \ref{Ex:red-bird}, it is perhaps surprising that Darwiche and Pearl never considered
postulate (P) below.
In this example, the argument for retaining the belief that the creature is red hinges
 upon the assumption that being red is not in conflict with the newly obtained information
 that it is a kind of animal. That is, because learning that the creature is an animal will not  
 automatically disqualify it from being red, it is reasonable to retain the belief that it is red.
 More generally then, whenever $\alpha$ is
consistent with a revision by $\beta$, it should be retained if an
$\alpha$-revision is inserted just before the $\beta$-revision.
 \begin{description}
\item[(P)] If $\lnot\alpha\notin B(\E\ast\beta)$ then $\alpha\in B(\E\ast\alpha\ast\beta)$
\end{description}
Applying (P) to Example \ref{Ex:red-bird} we see that, if $\mathtt{red}$
is consistent with $B(\E\ast\lnot\mathtt{bird})$, we have ${\mathtt red}\in B(\E\ast
\mathtt{red}\ast\lnot\mathtt{bird})$. Put differently, (P) requires that you retain your belief in the animal's redness, 
provided this would not have been precluded if the observation about it being red had never occurred. 
(P) was also proposed independently of the present paper by Jin and Thielscher \citeyear{Jin-ea:2005a} where it is
named \emph{Independence}.
The semantic counterpart of
(P) looks like this:
\begin{description}
\item[(PR)] For $v\in [\alpha]$ and $w\in [\lnot\alpha]$, if $v\preceq_{\E} w$ then
$v\prec_{\E\ast\alpha} w$
\end{description}
(PR) requires an $\alpha$-world $v$ that is at least as plausible as a $\lnot\alpha$-world $w$
to be strictly more plausible than $w$ after an $\alpha$-revision. The following result was
also proved independently by Jin and Thielscher \citeyear{Jin-ea:2005a}.
\begin{proposition}
\label{Prop:P-PR}%
If $\ast$ satisfies DP-AGM, then it satisfies (P) iff it also satisfies (PR).
\end{proposition}
\begin{proof}
For (P)$\Rightarrow$(PR), let $v\in[\alpha]$, $w\in[\lnot\alpha]$, $v\preceq_\E w$, 
and let $\beta$ be such that $[\beta]=\{v,w\}$. This means that $\lnot\alpha\notin B(\E*\beta)$
(since $[B(\E*\beta)]$ is either equal to $\{v\}$ or to $\{v,w\}$), and so, by (P), 
$\alpha\in B(\E*\alpha *\beta)$. And therefore $v\prec_{\E*\alpha}w$, for if not, 
we would have that $w\preceq_{\E*\alpha}v$, from which it follows 
that $w\in[B(\E*\alpha *\beta)]$, and so $\alpha\notin B(\E*\alpha *\beta)]$. 

For (P)$\Leftarrow$(PR), suppose that $\lnot\alpha\notin B(\E*\beta)$. This means 
there is a $v\in[\alpha]\cap[B(\E*\beta)]$; that is, $v\preceq_\E w$ for every $w\in[\beta]$. 
And this means that $\alpha\in B(\E*\alpha *\beta)$. For if not, it means there is an $x$ in 
$[\lnot\alpha]\cap [B(\E*\alpha *\beta)]$. Now, since $x\in[B(\E*\alpha *\beta)]$, it follows from DP-AGM
that $x\preceq_{\E*\alpha}w$ for every $w\in[\beta]$, and so $x\preceq_{\E*\alpha} v$ (since 
$v\in [B(\E\ast\beta)]\subseteq[\beta])$. 
But it also follows from DP-AGM that  $x\in[\beta]$, and therefore that $v\preceq_\E x$, and by (PR) it 
then follows that $v\prec_{\E*\alpha} x$; a contradiction.
\end{proof}

\noindent
Rule (PR) {\em enforces} certain changes in the ordering
$\preceq_{\E}$ after receipt of $\alpha$. 
In fact as soon as there
exist an $\alpha$-world $v$ and a $\neg\alpha$-world $w$ on the same
plausibility level somewhere in $\preceq_{\E}$ (in that both $v \preceq_\E w$ and $w \preceq_\E v$), (PR) implies
$\preceq_{\E*\alpha} \neq \preceq_{\E}$. 
Furthermore these changes
must also occur even when  $\alpha$ is already believed in $\E$ to
begin with, i.e., $\alpha \in B(\E)$. (Although of course if $\alpha
\in B(\E)$ then $B(\E * \alpha) = B(\E)$, i.e., the {\em knowledge base}
associated to $\E$ will remain unchanged -- this follows from DP-AGM.)
The rules (P)/(PR) ensure input $\alpha$ is believed with a certain
{\em minimal strength} of belief -- enough to help it survive the
{\em next} revision. The point that being informed of $\alpha$ can lead to an
increase in the strength of an agent's belief in $\alpha$, even in
cases where the agent already believes $\alpha$ to begin with, has
been made before, e.g., by Friedman and Halpern \citeyear[p.405]{Friedman-ea:99a}.
Note that (P) has the antecedent of (C4) and the consequent of (C3).
In fact, (P) is stronger than (C3) and (C4) combined. This is easily seen from
the semantic counterparts of these postulates. 
It also
follows that the only concrete example of an iterated revision
operator provided by Darwiche and Pearl, the operator they refer to
as $\bullet$ and which employs a form of Spohnian conditioning
\cite{Spohn:88a}, satisfies (PR), and therefore (P) as well.
Furthermore, by adopting (P) we explicitly exclude natural revision
as a permissible operator. So accepting (P) is a move towards the
viewpoint that information obtained before the latest input ought
not to be discarded unnecessarily.

Based on the analysis of this section we propose a strengthening of  the Darwiche-Pearl framework in which ($\E\ast$5) is replaced by ($\E\ast$5$'$) and
(C3) and (C4) are replaced by (P). 
\begin{definition}
A revision operator is \emph{admissible} iff it satisfies RAGM, (C1), (C2),
and (P).
\end{definition}
Inasmuch as  the Darwiche-Pearl framework can be visualised as one
in which $\alpha$-worlds slide ``downwards'' relative to
$\lnot\alpha$-worlds, admissible revision  ensures, via (PR), that this
``downwards'' slide is a strict one.

We now pave the way for the third postulate we would like to add in this paper to the Darwiche-Pearl framework. 
To begin with, note that another view of (P) is that it is a significant weakening of the
following property, first introduced by Nayak et al. \citeyear{Nayak-ea:96b}:
\begin{description}
\item[(Recalcitrance)] If $\lnot\alpha\notin Cn(\beta)$ then $\alpha\in B(\E\ast\alpha\ast\beta)$
\end{description}
Semantically, (Recalcitrance) corresponds to the following property, as was pointed out by
Booth \citeyear{Booth:2005a} and implicitly contained in the work of Nayak et al. \citeyear {Nayak-ea:2003a}:
\begin{description}
\item[(R)] For $v\in [\alpha]$, $w\in [\lnot\alpha]$, $v\prec_{\E\ast\alpha} w$
\end{description}
(Recalcitrance) is a property of the {\em lexicographic revision} operator, 
the second of the well-known DP-operators we consider, and one that is just as old as natural
revision. It was first introduced by Nayak \citeyear{Nayak:93a} and  has
been studied most notably by Nayak et al. \citeyear{Nayak:94b,Nayak-ea:2003a}, although, as with natural
revision, the idea actually dates back to Spohn \citeyear{Spohn:88a}.
In fact, lexicographic revision is characterised by DP-AGM (and also RAGM) together with (C1), (C2) and (Recalcitrance), a result that is
easily proved from the semantic counterparts of these properties and Nayak et al.'s semantic
characterisation of lexicographic revision in \citeyear{Nayak-ea:2003a}. Informally, lexicographic revision takes the assumption 
of ``most recent is best'', on which the Success postulate $(\E\ast 2)$ is based, and adds to it the assumption of temporal coherence. In combination, this leads to the stronger assumption that ``\emph{more} recent is better''.

An analysis of
the semantic characterisation of lexicographic revision shows that it is the \emph{least} conservative
of the DP-operators, in the sense that it effects the most changes in the relative ordering of valuations permitted by
DP-AGM (or RAGM for that matter) and the Darwiche-Pearl postulates. Since it is also an admissible revision operator, it follows that it is also the least conservative admissible operator.

The problem with (Recalcitrance) is that the decision of whether to accept $\beta$ after
a subsequent revision by $\alpha$ is completely determined by the logical relationship between $\beta$ and $\alpha$ -- the
epistemic state $\E$ is robbed of all influence. The replacement of (Recalcitrance)
by the weaker (P) already gives $\E$ more influence in the outcome. What we will do shortly is
constrain matters further by giving $\E$ as much influence as allowed by the postulates for
admissible revision. Such a move ensures greater sensitivity to the agent's epistemic record in making further changes.

Note that lexicographic revision assumes that more recent
information takes complete precedence over information obtained
previously. Thus, when applied to Example \ref{Ex:red-bird}, it
requires us to believe that the animal, previously assumed to be a
bird, is indeed red, because ${\mathtt red}$ is a recent input which does not
conflict with the most recently obtained input. While this is a
reasonable approach in many circumstances, a dogmatic adherence to
it can be problematic, as the following example shows.
\begin{example}
\label{Ex:red-bird-2} While holidaying in a wildlife park we observe a creature which is clearly red,
but we are too far away to determine whether it is a  bird or a land
animal. So we adopt the knowledge base $B(\E)= Cn(\mathtt{red})$. Next to us
is a person with knowledge of the local area who declares that, since the creature is red, it is a 
a bird. We have no reason to doubt him, and so we adopt the belief
$\mathtt{red\rightarrow bird}$. Now the creature moves closer and it becomes
clear that it is not a bird. The question is, should we continue
believing that it is red? Under the circumstances described above we
want our initial observation to take precedence, and believe that the animal is red. But
lexicographic revision does not allow us to do so.
\end{example}
Other examples along similar lines speaking against a rigid acceptance of (Recalcitrance) are those
of Glaister \citeyear[p.31]{Glaister:98a} and Jin and Thielscher \citeyear[p.482]{Jin-ea:2005a}.

While (P) allows for the possibility of retaining the belief that the animal is red, it
does not \emph{enforce} this belief.  The rest of this section is devoted to the 
discussion of a property which does so.
To help us express this property, we introduce an extra piece of terminology
and notation.



\begin{definition}
$\alpha$ and $\beta$ \emph{counteract} with respect to an epistemic state $\E$, written
$\alpha \counteracts_{\E} \beta$, iff $\lnot\beta\in B(\E\ast\alpha)$ and $\lnot\alpha\in B(\E\ast\beta)$.
\end{definition}
The use of the term {\em counteract} to describe this relation is
taken from Nayak et al. \citeyear{Nayak-ea:2003a}. \mbox{$\alpha \counteracts_{\E}
\beta$} means that, from the viewpoint of $\E$, $\alpha$ and $\beta$
tend to ``exclude'' each other. We will now discuss a few properties of this relation. First note that $\counteracts_{\E}$ depends only on the total preorder $\preceq_{\E}$
obtained from $\E$. Indeed we have $\alpha \counteracts_{\E} \beta$ iff both $\min(\alpha, \preceq_\E) \subseteq
[\neg\beta]$ and $\min(\beta, \preceq_\E) \subseteq [\neg\alpha]$. This in turn can be reformulated in the following
way, which provides a useful aid to visualise a counteracts relation: 

\begin{proposition}
\label{counter}
$\alpha \counteracts_\E \beta$ iff there exist $v \in [\alpha]$, $w \in [\beta]$ such that both $v \prec_\E x$ and
$w \prec_\E x$ for all $x \in \min(\alpha \wedge \beta, \preceq_\E)$.
\end{proposition}

\begin{proof}
First note that, since obviously $\min(\alpha, \preceq_\E) \subseteq [\alpha]$, $\min(\alpha, \preceq_\E) \subseteq [\neg\beta]$ may be rewritten as $\min(\alpha, \preceq_\E) \subseteq [\neg(\alpha \wedge \beta)]$. Using the fact that
$\preceq_\E$ is a total preorder, it is easy to see that
this can hold
iff there exists $v \in [\alpha]$ such that $v \prec_\E x$ for all $x \in \min(\alpha \wedge \beta, \preceq_\E)$. In the
same way we may rewrite $\min(\beta, \preceq_\E) \subseteq [\neg\alpha]$ as $\min(\beta, \preceq_\E) \subseteq [\neg(\alpha \wedge \beta)]$, which is then equivalent to saying
there exists $w \in [\neg\beta]$ such that $w \prec_\E x$ for all $x \in \min(\alpha \wedge \beta, \preceq_\E)$.    
\end{proof}

\noindent 
In other words, Proposition \ref{counter} says $\alpha \counteracts_\E \beta$ iff there exist {\em both} an $\alpha$-world
and a $\beta$-world which are strictly more plausible than the most plausible $(\alpha \wedge \beta)$-worlds.
Other immediate things to note about $\counteracts_\E$ are that it is symmetric, and that it is syntax-independent, i.e.,
if $\alpha \counteracts_\E \beta$ and $\beta \equiv \beta'$ then $\alpha \counteracts_\E \beta'$. Furthermore if $\alpha$ 
and $\beta$ are
logically inconsistent with each other then $\alpha \counteracts_{\E} \beta$, but the
converse need not hold (see the short example after the next proposition for confirmation). Thus $\counteracts_{\E}$ can 
be seen as a weak form of inconsistency. The next result gives two more properties of $\counteracts_\E$:

\begin{proposition}
\label{counteractsprops}
Given RAGM, the following properties hold for $\counteracts_\E$: 
\\
{\em (i)}
If $\alpha \counteracts_{\E} \beta$ and $\gamma \counteracts_{\E} \beta$ then $(\alpha \vee \gamma) \counteracts_{\E} 
\beta$ \\
{\em (ii)}
If $\alpha \not\counteracts_\E \beta$ and $\gamma \ncounteracts_{\E} \beta$ then $(\alpha \vee \gamma) \not\counteracts_{\E} 
\beta$
\end{proposition}

\begin{proof}
{\em (i)} 
Suppose $\alpha \counteracts_{\E} \beta$ and $\gamma \counteracts_{\E} \beta$. To show $(\alpha \vee \gamma) \counteracts_{\E} 
\beta$ we need to show both $\neg\beta \in B(\E * (\alpha \vee \gamma))$ and $\neg(\alpha \vee \gamma) \in B(\E * \beta)$. For the former we already have both $\neg\beta \in B(\E * \alpha)$ and $\neg\beta \in B(\E * \gamma)$ from $\alpha \counteracts_{\E} \beta$ and $\gamma \counteracts_{\E} \beta$ respectively. Since it follows from RAGM that $B(\E * \lambda) \cap
B(\E * \chi) \subseteq B(\E * (\lambda \vee \chi))$ for {\em any} $\lambda, \chi \in L$, we 
can conclude from this $\neg\beta \in B(\E * (\alpha \vee \gamma))$. For the latter we already have both $\neg\alpha \in B(\E * \beta)$ and $\neg\gamma \in B(\E * \beta)$ from
$\alpha \counteracts_{\E} \beta$ and $\gamma \counteracts_{\E} \beta$ respectively. And from this we can conclude $\neg(\alpha \vee \gamma) \in B(\E * \beta)$, again using RAGM (specifically ($\E\ast 1$)). 
\\
{\em (ii)}
Suppose  $\alpha \not\counteracts_\E \beta$ and $\gamma \ncounteracts_{\E} \beta$. Firstly, if either $\neg\alpha \not\in
B(\E * \beta)$ or $\neg\gamma \not\in B(\E * \beta)$ then we must have $\neg(\alpha \vee \gamma) \not\in B(\E * \beta)$ by 
RAGM and so $(\alpha \vee \gamma) \not\counteracts_{\E} 
\beta$ as required. So suppose both $\neg\alpha \in
B(\E * \beta)$ and $\neg\gamma \in B(\E * \beta)$. Then, since $\alpha \not\counteracts_\E \beta$ and $\gamma \ncounteracts_{\E} \beta$, this means we have both $\neg\beta \not\in B(\E * \alpha)$ and $\neg\beta \not\in B(\E * \gamma)$. Since it follows from RAGM that 
$B(\E * (\lambda \vee \chi)) \subseteq B(\E * \lambda) \cup
B(\E * \chi)$ for {\em any} $\lambda, \chi \in L$, it
follows from these two that $\neg\beta \not\in B(\E * (\alpha \vee \gamma))$ and so also in this case $(\alpha \vee \gamma) \not\counteracts_{\E} 
\beta$ as required.
\end{proof}

\noindent
The first property above says that if $\beta$ counteracts with two sentences separately, then it counteracts
with their disjunction, while the second says that it cannot counteract with a disjunction without counteracting with
{\em at least one} of the disjuncts.
Obviously these properties also hold for the binary relation of logical inconsistency. However one departure from
the inconsistency relation is that it is possible to have both $\gamma \not\counteracts_\E \beta$ and 
$(\alpha \vee \gamma) \counteracts_\E \beta$. To see this assume for the moment $L$ is generated by just three propositional
atoms $\{p,q,r\}$ and take $\alpha = p$, $\beta = q$ and $\gamma = r$. Then take 
$\preceq_\E$ to be such that its lowest plausibility level contains only the two valuations
$010$ and $100$, and the next plausibility level only the valuation $111$.

We are now ready to introduce our third postulate. It is the following:
\begin{description}
\item[(D)] If $\alpha\counteracts_{\E}\beta$ then
                   $\lnot\alpha\in B(\E\ast\alpha\ast\beta)$
\end{description}
(D) requires that, whenever $\alpha$ and $\beta$ counteract with
respect to $\E$, $\alpha$ should be \emph{disallowed} when an
$\alpha$-revision is followed by a $\beta$-revision. That is, when
the $\beta$-revision of $\E\ast\alpha$ takes place, the information
encoded in $\E$ takes precedence over the information contained in
$\E\ast\alpha$. Darwiche and Pearl  \citeyear{Darwiche-ea:97a} considered this property (it is
their rule (C6)) but argued against it,
citing the following example.
\begin{example}
\label{Ex:john-mary}%
\cite{Darwiche-ea:97a} We believe that exactly one of John and Mary
committed a murder. Now we get persuasive evidence indicating that
John is the murderer. This is followed by persuasive information
indicating that Mary is the murderer.  Let $\alpha$ represent that
John committed the murder and $\beta$ that Mary committed the
murder. Then (D) forces us to conclude that Mary, but not John, was
involved in the murder.  This, according to Darwiche and Pearl, is
counterintuitive, since we should conclude that both were involved
in committing the murder.
\end{example}
Darwiche and Pearl's argument against (D) rests upon the assumption that more recent information
ought to take precedence over information previously obtained. But as we have seen in Example
\ref{Ex:red-bird-2}, this is not always a valid assumption. In fact, the application of (D) to
Example \ref{Ex:red-bird-2}, with $\alpha=\mathtt{red\rightarrow bird}$ and $\beta=\mathtt{\lnot bird}$, produces the
intuitively correct result of a belief in the observed animal being red:
${\mathtt red}\in B(\E\ast(\mathtt{red\rightarrow bird)\ast\lnot bird}$).

Another way to gain insight into the significance of (D) is to consider its semantic counterpart:
\begin{description}
\item[(DR)] For $v\in [\lnot\alpha]$, $w\in [\alpha]$, and $w\notin [B(\E\ast\alpha)]$, if
$v\prec_{\E} w$ then $v\prec_{\E\ast\alpha} w$
\end{description}
(DR) curtails the rise in plausiblity of $\alpha$-worlds after an $\alpha$-revision. It ensures that,
with the exception of the most plausible $\alpha$-worlds, the relative ordering between an
$\alpha$-world and the $\lnot\alpha$-worlds more plausible than it remains unchanged.
\begin{proposition}
\label{Prop:D-DR}%
Whenever a revision operator $\ast$ satisfies RAGM, then $\ast$ satisfies (D) iff it satisfies (DR).
\end{proposition}
\begin{proof}
For (D)$\Rightarrow$(DR), suppose that $v\in[\lnot\alpha]$, $w\in[\alpha]$, $w\notin[B(\E*\alpha)]$, 
$v\prec_\E w$, and let $\beta$ be such that $[\beta]=\{v,w\}$. Then $\lnot\alpha\in B(\E*\beta)$
and $\lnot\beta\in B(\E*\alpha)$, and so, by (D). $\lnot\alpha\in B(\E*\alpha *\beta)$. 
From this it follows that $v\prec_{\E*\alpha} w$. For if not, we would have that 
$w\preceq_{\E*\alpha} v$, which means that $w\in[B(\E*\alpha *\beta)]$, and therefore that
$\lnot\alpha\notin B(\E*\alpha *\beta)$; a contradiction.

For (D)$\Leftarrow$(DR), suppose that $\lnot\beta\in B(\E*\alpha)$ and that 
 $\lnot\alpha\in B(\E*\beta)$, but assume that $\lnot\alpha\notin B(\E*\alpha *\beta)$.
 This means there is a $w\in[\alpha]$ that is also in  $[B(\E*\alpha *\beta)]$. 
 Now observe that $w\notin [B(\E*\beta)]$  since $w$ is an $\alpha$-model. 
 Also, since $w\in [B(\E*\alpha *\beta)]$, it follows from RAGM that $w$ is a $\beta$-model, 
 and therefore $w\notin[B(\E*\alpha)]$.  Our supposition of 
 $\lnot\alpha\in B(\E\ast\beta)$ means that $\min(\beta,\preceq_\E)\subseteq [\lnot\alpha]$.
 Since $w\in[\alpha]\cap[\beta]$ it thus follows that there is a $v\in[\beta]\cap[\lnot\alpha]$ such that $v\prec_\E w$. 
  By (DR) it then 
 follows that $v\prec_{\E*\alpha}w$. But then $w$ cannot be a model of $B(\E*\alpha *\beta)$;
 a contradiction. 
\end{proof}

\section{Restrained Revision}
\label{Sec:restrained}%
We now strengthen the requirements on admissible revision (those
operators satisfying RAGM, (C1), (C2) and (P)) by insisting that (D)
is satisfied as well. To do so, let us first consider the semantic
definition of an interesting admissible revision operator. Recall
that RAGM fixes the set of ($\preceq_{\E\ast\alpha}$)-minimal
models, setting them equal to $\min(\alpha,\preceq_{\E})$, but places
no restriction on how the remaining valuations should be ordered.
The following property provides a \emph{unique} relative ordering of
the remaining valuations.
\begin{description}
\item[(RR)] $\forall v,w\notin [B(\E\ast\alpha)]$, $v\preceq_{\E\ast\alpha} w$ iff
$\left\{%
\begin{array}{l}
v\prec_{\E} w\textrm{ or,} \\
v\preceq_{\E} w\textrm{ and }(v\in [\alpha]\textrm{ or  }w\in [\lnot\alpha])
\end{array}
\right.$
\end{description}
(RR) says that the relative ordering of the
valuations that are not ($\preceq_{\E\ast\alpha}$)-minimal remains unchanged, except
for $\alpha$-worlds and $\lnot\alpha$-worlds on the same plausibility level; those are split
into two levels with the $\alpha$-worlds more plausible than the $\lnot\alpha$-worlds.
So RAGM combined with (RR) fixes a unique ordering of valuations.
\begin{definition}
The revision operator satisfying RAGM and (RR) is called \emph{restrained revision}.
\end{definition}
It turns out that within the framework provided by admissible revision, it is only restrained
revision that satisfies (D). 
We prove this with the help of the following lemma, asserting the equivalence of 
(RR) to (CR1), (CR2), (PR) and (DR) in the presence of RAGM.
\begin{lemma}
\label{Lem:R}%
Whenever a revision operator $*$ satisfies RAGM, then $*$ satisfies (RR) iff it satisfies
(CR1), (CR2), (PR), and (DR). 
\end{lemma}
\begin{proof}
For (CR1)$\Leftarrow$(RR), pick  $v,w\in[\alpha]$. If  $v\in[B(\E\ast\alpha)]$ then 
(CR1) follows from RAGM. If not, it follows from RAGM that $w\notin[B(\E\ast\alpha)]$, 
and then (CR1) follows from a direct application of (RR). For (CR2)$\Leftarrow$(RR), pick  $v,w\in[\lnot\alpha]$.
From RAGM it follows that  $v,w\notin[B(\E\ast\alpha)]$, and then we obtain (CR2) from a direct
application of (RR). 

Observe that (RR) can be rewritten as 
\begin{description}
\item[(RR$'$)] $\forall v,w\notin [B(\E\ast\alpha)]$, $v\prec_{\E\ast\alpha} w$ iff
$\left\{%
\begin{array}{l}
v\preceq_{\E} w\textrm{ and,} \\
v\prec_{\E} w\textrm{ or }(v\in [\alpha]\textrm{ and  }w\in [\lnot\alpha])
\end{array}
\right.$
\end{description}
Now, for (PR)$\Leftarrow$(RR), pick  $v\in[\alpha]$ and  $w\in[\lnot\alpha]$.
 If $v\in[B(\E\ast\alpha)]$ then (PR) follows from RAGM. If not, it follows from 
 a direct application of (RR$'$). For (DR)$\Leftarrow$(RR), pick  $v\in[\lnot\alpha]$, $w\in[\alpha]$,
and  $w\notin[B(\E\ast\alpha)]$. Then (PR) follows from a direct application of (RR$'$). 

For (CR1), (CR2), (PR), (DR)$\Rightarrow$(RR),  let $v,w\notin [B(\E*\alpha)]$
and suppose that $v\preceq_{\E\ast\alpha} w$ and $v\not\prec_\E w$ (i.e. $w\preceq_{\E} v$).
We have to show  that $v\preceq_{\E} w$ and either $v\in[\alpha]$ or $w\in[\lnot\alpha]$. 
Assume this is not the case. Then  $w\prec_{\E} v$ or both $v\in[\lnot\alpha]$ and $w\in[\alpha]$.  
Now, the second case is impossible because, together with $w\preceq_{\E} v$ and (PR) it 
implies that $w\prec_{\E\ast\alpha} v$;  a contradiction. But the first case is also impossible.
To see why, observe that by (CR1) it implies that  $v$ and $w$ cannot both be $\alpha$-models, by (CR2) $v$ and $w$ cannot both be $\lnot\alpha$-models, by (DR) it cannot be 
the case that $w\in[\lnot\alpha]$ and $v\in[\alpha]$. And by (PR) it cannot be the case
that $w\in[\alpha]$ and $v\in[\lnot\alpha]$. This concludes the first part of the proof of (CR1), (CR2), (PR), (DR)$\Rightarrow$(RR). For the second part, let $v,w\notin [B(\E*\alpha)]$ and suppose first that 
$v\prec_{\E}w$. If $v,w\in[\alpha]$ then $v\preceq_{\E\ast\alpha}w$ follows from (CR1). 
 If $v,w\in[\lnot\alpha]$ then $v\preceq_{\E\ast\alpha}w$ follows from (CR2).
 If $v\in[\alpha]$ and $w\in[\lnot\alpha]$ then $v\preceq_{\E\ast\alpha}w$ follows from (PR).
 If $v\in[\lnot\alpha]$ and $w\in[\alpha]$ then $v\preceq_{\E\ast\alpha}w$ follows from (DR).
 Now suppose that $v\preceq_\E w$ and either $v\in[\alpha]$ or $w\in[\lnot\alpha]$. 
If $v\in[\alpha]$ then $v\preceq_{\E\ast\alpha}w$ follows either from (CR1) or (PR), depending
on whether $w\in[\alpha]$ or $w\in[\lnot\alpha]$. And similarly, if $w\in[\lnot\alpha]$ then $v\preceq_{\E\ast\alpha}w$ follows either from (CR2) or (PR), depending
on whether $v\in[\lnot\alpha]$ or $v\in[\alpha]$.
\end{proof}

\begin{theorem}
\label{Thm:restrained}%
RAGM, (C1), (C2), (P) and (D) provide an exact characterisation of restrained revision.
\end{theorem}
\begin{proof}
The proof follows from Lemma \ref{Lem:R}, Proposition \ref{Prop:P-PR}, Proposition \ref{Prop:D-DR}, 
and the correspondence between (C1) and (CR1), and (C2) and (CR2). 
\end{proof}

\noindent Another interpretation of  (RR) is that it maintains the relative
ordering of the valuations that are not
($\preceq_{\E\ast\alpha}$)-minimal, except for the changes mandated
by (PR). From this it can be seen that restrained revision is the
most conservative of all admissible revision operators, in the sense that 
effects the least changes in the relative ordering of valuations permitted by
admissible revision. So, in the
context of admissible revision, restrained revision takes on the
role played by natural revision in the Darwiche-Pearl framework.

In the rest of this section we examine some further properties of restrained revision. Firstly,
Examples \ref{Ex:red-bird-2} and \ref{Ex:john-mary} share some
interesting structural properties. In both, the initial knowledge
base $B(\E)$ is pairwise consistent with each of the subsequent
sentences in the revision sequence, while the sentences in each
revision sequence are pairwise inconsistent.  And in both examples
the information contained in the initial knowledge base $B(\E)$ is
retained after the revision sequence.  These commonalities are
instances of an important general result. Let $\Gamma$ denote the
non-empty sequence of inputs $\gamma_1,\ldots,\gamma_n$, and let
$\E\ast\Gamma$ denote the revision sequence
$\E\ast\gamma_1\ast\ldots\ast\gamma_n$. Furthermore we shall refer
to an epistemic state $\E$ as $\Gamma$\emph{-compatible} provided
that $\lnot\gamma_i\notin B(\E)$ for every $i$ in $\{1,\ldots,n\}$.
\begin{description}
\item[(O)] If  $\E$ is $\Gamma$-compatible then $B(\E)\subseteq B(\E\ast\Gamma)$
\end{description}
(O) says that as long as $B(\E)$ is not in direct conflict with \emph{any} of the inputs in the
sequence $\gamma_1,\ldots,\gamma_n$, the entire $B(\E)$  \emph{has} to be
propagated to the knowledge base obtained from the revision sequence  $\E\ast\gamma_1\ast\ldots \ast\gamma_n$.
This is a preservation property that is satisfied by restrained revision.
\begin{proposition}
\label{Prop:O}%
Restrained revision satisfies (O).
\end{proposition}
\begin{proof}
We denote by $\E\ast\Gamma_i$, for $i=0,\ldots,n$, the revision sequence 
$\E\ast\gamma_1,\ldots,\gamma_i$ (with $\E\ast\Gamma_0=\E$).
We give an inductive proof that, for $\forall v\in[B(\E)]$ and  $\forall w\notin[B(\E)]$,
$v\prec_{\E\ast\Gamma_i}w$ for $i=0,\ldots,n$. In other words, every $B(\E)$-world 
is always strictly below every non $B(\E)$-world. From this the result follows
immediately. For $i=0$ this amounts to showing that
$v\prec_{\E} w$ which follows immediately from the definition of $\preceq_{\E}$ and $B(\E)$.
Now pick any $i=1,\ldots,n$ and assume that $v\prec_{\E\ast\Gamma_{i-1}}w$. 
We consider four cases. If $v,w\in[\gamma_i]$ then it follows by (CR1) that 
$v\prec_{\E\ast\Gamma_i}w$.  If $v,w\in[\lnot\gamma_i]$ then it follows by (CR2) that 
$v\prec_{\E\ast\Gamma_i}w$.  If $v\in[\gamma_i]$ and $w\in[\lnot\gamma_i]$ then it follows 
by (PR) that  $v\prec_{\E\ast\Gamma_i}w$. And finally, suppose $v\in[\lnot\gamma_i]$ and $w\in[\gamma_i]$. By $\Gamma_i$-compatibility there is an $x\in[B(E)]\cap[\gamma_i]$, 
and by the inductive hypothesis, $x\prec_{\E\ast\Gamma_{i-1}}w$. So $w\notin[B(\E\ast\Gamma_i)]$, and then it follows by (DR) that  $v\prec_{\E\ast\Gamma_i}w$. 
 \end{proof}

\noindent Although restrained revision preserves information which has not been directly contradicted, it is
not dogmatically wedded to older information.  If neither of two successive, but incompatible,
epistemic states are in conflict with any of the inputs of a sequence
$\Gamma = \gamma_1,\ldots,\gamma_n$,
it prefers the latter epistemic state when revising by $\Gamma$.
\begin{proposition}
Restrained revision satisfies the following property:
\begin{description}
\item[(Q)] If $\E$ and $\E\ast\alpha$ are both $\Gamma$-compatible but
$B(\E)\cup B(\E\ast\alpha)\vDash\bot$, then $B(\E*\alpha)\subseteq B(\E\ast\alpha\ast\Gamma)$
and $B(\E)\nsubseteq B(\E\ast\alpha\ast\Gamma)$
\end{description}
\end{proposition}
\begin{proof}
It follows immediately from Proposition \ref{Prop:O} that 
$B(\E*\alpha)\subseteq B(\E\ast\alpha\ast\Gamma)$. And $B(\E)\nsubseteq B(\E\ast\alpha\ast\Gamma)$
then follows from the consistency of $B(\E\ast\alpha\ast\Gamma)$.
\end{proof}

\noindent Next we consider another preservation property, but this time, unlike the case for (O) and (Q), we
look at circumstances where $B(\E)$ is incompatible with some of the inputs in a revision sequence.
\begin{description}
\item[(S)]
If $\neg\beta \in B(\E * \alpha)$ and $\neg\beta \in B(\E * \neg\alpha)$ then
$B(\E * \alpha * \neg\alpha * \beta) = B(\E * \alpha * \beta)$
\end{description}
Note that, given RAGM, the antecedent of (S) implies that $\neg\beta
\in B(\E)$. Thus (S) states that if $\neg\beta$ is believed
initially, and that a subsequent commitment to either $\alpha$ or
its negation would not change this fact, then after the sequence of
inputs in which $\beta$ is preceded by $\alpha$ and $\neg\alpha$,
the {\em second} input concerning $\alpha$ is nullified, and the
older input regarding $\alpha$ is retained.
\begin{proposition}
Restrained revision satisfies (S).
\end{proposition}
\begin{proof}
Suppose the antecedent holds. If $\neg\alpha \counteracts_{\E*\alpha} \beta$ then the consequent holds. In fact this
can be seen from the property (T) in Proposition \ref{succinct} below. So
suppose $\neg\alpha \not\counteracts_{\E*\alpha} \beta$. Then either $\alpha \not\in B(\E * \alpha * \beta)$ or
$\neg\beta \not\in B(\E * \alpha * \neg\alpha)$. This latter doesn't hold by one of the assumptions together with (C2), so the former must hold. This implies $\neg\alpha \in B(\E * \beta)$ by (P). Combining this with the other assumption we
get $\alpha \counteracts_{\E} \beta$. In this case we get $B(\E * \alpha * \beta) = B(\E * \beta)$ (again using
(T)), while (since
$\neg\alpha \not\counteracts_{\E*\alpha} \beta$) $B(\E * \alpha * \neg\alpha * \beta) = B(\E * \alpha * (\neg\alpha \wedge \beta))$ ((T) once more), which in turn equals $B(\E * (\neg\alpha \wedge \beta))$ by (C2). Since $\neg\alpha \in B(\E * \beta)$ this is in turn equal to
$B(\E * \beta)$ by RAGM as required.
\end{proof}

\noindent
We now provide a more compact syntactic representation of restrained
revision. First we show that (C1) and (P) can be combined into a
single property, and so can (C2) and (D).
\begin{proposition}
\label{compact}
Given RAGM,
\begin{enumerate}
\item (C1) and (P) are together equivalent to the single rule
\begin{description}
\item[(C1P)] If $\neg\alpha \not\in B(\E \ast \beta)$ then $B(\E \ast \alpha \ast \beta) = B(\E \ast (\alpha \wedge \beta))$
\end{description}
\item (C2) and (D) are together equivalent to the single rule
\begin{description}
\item[(C2D)] If $\alpha \counteracts_{\E} \beta$ then $B(\E \ast \alpha \ast \beta) = B(\E \ast \beta)$.
\end{description}
\end{enumerate}
\end{proposition}
\begin{proof}
For (C1),(P)$\Rightarrow$(C1P), suppose  $\neg\alpha \not\in B(\E \ast \beta)$. 
By (P) it follows that $\alpha\in B(\E\ast\alpha\ast\beta)$ which means, by RAGM, 
that $B(\E\ast\alpha\ast\beta)=B(\E\ast\alpha\ast(\alpha\wedge\beta))$. By (C1) it follows
that $B(\E\ast\alpha\ast(\alpha\wedge\beta))=B(\E\ast(\alpha\wedge\beta))$, and thus
that $B(\E\ast\alpha\ast\beta)=B(\E\ast(\alpha\wedge\beta))$.
For (C1)$\Leftarrow$(C1P), suppose that $\beta\vDash\alpha$. Then  
$\neg\alpha \not\in B(\E \ast \beta)$ by RAGM, and so $B(\E\ast\alpha\ast\beta)=B(\E\ast(\alpha\wedge\beta))$ by (C1P).
But since $\beta\equiv\alpha\wedge\beta$ it follows that $B(\E\ast\alpha\ast\beta)=B(\E\ast\beta)$.
For (P)$\Leftarrow$(C1P), suppose that $\neg\alpha \not\in B(\E \ast \beta)$.
Then $B(\E\ast\alpha\ast\beta)=B(\E\ast(\alpha\wedge\beta))$ by (C1P) which means, by RAGM, 
that $\alpha\in B(\E\ast\alpha\ast\beta)$.

For (C2),(D)$\Rightarrow$(C2D), suppose that $\alpha \counteracts_{\E} \beta$. By (D), 
$\lnot\alpha\in B(\E\ast\alpha\ast\beta)$. By RAGM this means that
$B(\E\ast\alpha\ast\beta)=B(\E\ast\alpha\ast(\lnot\alpha\wedge\beta))$.
Now, by (C2) it follows that
$B(\E\ast\alpha\ast(\lnot\alpha\wedge\beta))=B(\E\ast(\lnot\alpha\wedge\beta))$. 
So $B(\E\ast\alpha\ast\beta)=B(\E\ast(\lnot\alpha\wedge\beta))$.
But since $\lnot\alpha\in B(\E\ast\beta)$, we get by RAGM that 
$B(\E\ast(\lnot\alpha\wedge\beta))=B(\E\ast\beta)$, from which it follows that
$B(\E\ast\alpha\ast\beta)=B(\E\ast\beta)$. 
For (C2)$\Leftarrow$(C2D), suppose that $\beta\vDash\lnot\alpha$. Then
$\alpha \counteracts_{\E} \beta$ for any $\E$ and by $B(\E\ast\alpha\ast\beta)=B(\E\ast\beta)$
by (C2D). For (D)$\Leftarrow$(C2D), suppose that 
$\alpha \counteracts_{\E} \beta$. Then $B(\E\ast\alpha\ast\beta)=B(\E\ast\beta)$ by (C2D) and since
$\lnot\alpha\in B(\E\ast\beta)$, it follows that $\lnot\alpha\in B(\E\ast\alpha\ast\beta)$.
\end{proof}

\noindent Both (C1P) and (C2D) provide conditions for the reduction of the two-step revision sequence
$\E\ast\alpha\ast\beta$ to a single-step revision (if only as regards the resulting knowledge base). (C1P)  reduces it to an ($\alpha\wedge\beta$)-revision when $\alpha$ is consistent with a $\beta$-revision.
(C2D) reduces it to a $\beta$-revision, ignoring $\alpha$ completely,
when $\alpha$ and $\beta$ counteract with respect to $\E$.
Now, it follows from RAGM that  the consequent of (C1P) also obtains when
$\neg\beta \not\in B(\E \ast \alpha)$. Putting this together we get a most succinct characterisation of
restrained revision.
\begin{proposition}
\label{succinct}%
Only restrained revision satisfies RAGM and:
\begin{description}
\item[(T)] $B(\E \ast \alpha \ast \beta) =
\left\{
\begin{array}{ll}
B(\E \ast \beta) & \textrm{if } \alpha \counteracts_{\E} \beta \\
B(\E \ast (\alpha \wedge \beta)) & \textrm{otherwise.}
\end{array}
\right. $
\end{description}
\end{proposition}
\begin{proof}
From Theorem \ref{Thm:restrained} and Proposition \ref{compact} it is sufficient
to show that RAGM, (C1P) and (C2D) hold iff RAGM and (T) hold. So, suppose
that $\ast$ satisfies RAGM and (T). (C1P) follows from the bottom part of (T), 
while (C2D) follows from the top part. Conversely, suppose that $\ast$ 
satisfies RAGM, (C1P) and (C2D). If $\alpha \counteracts_{\E} \beta$ it follows
from (C2D) that $B(\E\ast\alpha\ast\beta)=B(\E\ast\beta)$. If not, we consider two 
cases. If $\lnot\alpha\notin B(\E\ast\beta)$  it follows from (C1P) that 
$B(\E\ast\alpha\ast\beta)=B(\E\ast(\alpha\wedge\beta))$.  
Otherwise it has to be the case that  $\lnot\beta\notin B(\E\ast\alpha)$.
But then it follows from RAGM that $B(\E\ast\alpha\ast\beta)=B(\E\ast(\alpha\wedge\beta))$.  
\end{proof}

\noindent If we were to replace ``$\alpha \counteracts_{\E} \beta$'' in the
first clause in (T) by the stronger ``$\alpha$ and $\beta$ are
logically inconsistent'', we would obtain instead the
characterisation of lexicographic revision given by Nayak et al.
\citeyear{Nayak-ea:2003a}.

Proposition \ref{succinct} allows us to see clearly another significant property of restrained
revision. For if $\alpha \counteracts_\E \beta$ then we know $\neg\alpha \in B(\E * \alpha * \beta)$
directly from (D), while if $\alpha \not\counteracts_\E \beta$ then Proposition \ref{succinct} tells us
$B(\E * \alpha * \beta) = B(\E * (\alpha \wedge \beta))$
and so $\alpha \in B(\E * \alpha * \beta)$ by RAGM. Thus we see in the state $\E * \alpha * \beta$ the epistemic
status of $\alpha$ (either accepted or rejected) is {\em always} completely determined, i.e., we have proved:
\begin{proposition}
Restrained revision satisfies the following property:
\begin{description}
\item[(U)] If  $\neg\alpha \not\in B(\E * \alpha * \beta)$ then $\alpha \in B(\E * \alpha * \beta)$
\end{description}
\end{proposition}
(Given its similar characterisation just mentioned above, it is easy to see lexicographic revision satisfies
(U) too.) Like (P), property (U) can be read as providing conditions under which the penultimate
revision input $\alpha$ should be believed. Its antecedent is simply saying $B(\E * \alpha * \beta)$ is
{\em consistent} with $\alpha$. Thus (U) is saying the penultimate input should be believed {\em as long as it is 
consistent to do so}. By chaining (U) together with (C4), we easily see that (U) actually {\em implies} (P)
in the presence of (C4). As a consequence, we obtain the following alternative axiomatic characterisation of restrained
revision.
\begin{theorem}
RAGM, (C1), (C2), (C4), (U) and (D) provide an exact characterisation of restrained revision.
\end{theorem}
For (U), we are also able to provide a simple semantic counterpart property. It corresponds
to a {\em separating} of all the $\alpha$-worlds from all the $\neg\alpha$-worlds in the
total preorder $\preceq_{\E * \alpha}$ following an $\alpha$-revision, in that each
plausibility level in $\preceq_{\E * \alpha}$ either contains only $\alpha$-worlds
or contains only $\neg\alpha$-worlds:
\begin{proposition}
Whenever a revision operator $*$ satisfies RAGM, then $*$ satisfies (U) iff it satisfies the following
property: 
\begin{description}
\item[(UR)] For $v \in [\alpha]$ and $w \in [\neg\alpha]$, either $v \prec_{\E * \alpha} w$ or $w \prec_{\E *\alpha} v$
\end{description}
\end{proposition}
\begin{proof}
For (U)$\Rightarrow$(UR) suppose (UR) doesn't hold, i.e., there exist $\alpha$, $v \in [\alpha]$ and
$w \in [\neg\alpha]$ such that both $v \preceq_{\E * \alpha} w$ and $w \preceq_{\E * \alpha} v$. Letting
$\beta$ be such that $[\beta] = \{v,w\}$ we get $[B(\E * \alpha * \beta)] = \{v,w\}$ from RAGM and thus both
$\neg\alpha, \alpha \not\in B(\E * \alpha * \beta)$ (because of $v,w \in [B(\E * \alpha * \beta)]$
respectively). Hence (U) doesn't hold.

For (U)$\Leftarrow$(UR) suppose (U) doesn't hold, i.e., there exist $\alpha, \beta$ such that both
$\neg\alpha, \alpha \not\in B(\E * \alpha * \beta)$. Then there exist $v \in [\alpha]$ and $w \in [\neg\alpha]$
such that $v,w \in [B(\E * \alpha * \beta)] =$ (by RAGM) $\min(\beta, \preceq_{\E * \alpha})$. Since both $v$ and
$w$ are $(\preceq_{\E * \alpha})$-minimal $\beta$-worlds we must have both $v \preceq_{\E * \alpha} w$ and
$w \preceq_{\E * \alpha} v$. Hence $\alpha, v, w$ give a counterexample to (UR).
\end{proof}

Finally in this section we turn to two properties first mentioned (as far as we know) 
by Schlecta et al. \citeyear{Schlechta-ea:96a} (see also the work of Lehmann et al. \citeyear{Lehmann-ea:2001a}):
\begin{description}
\item[(Disj1)]
$B(\E \ast \alpha \ast \beta) \cap B(\E \ast \gamma \ast \beta) \subseteq B(\E \ast (\alpha \vee \gamma) \ast \beta)$
\item[(Disj2)]
$B(\E \ast (\alpha \vee \gamma) \ast \beta) \subseteq B(\E \ast \alpha \ast \beta) \cup B(\E \ast \gamma \ast \beta)$
\end{description}
(Disj1) says that if a sentence is believed after any one of two sequences of revisions that
differ only at step $i$ (step $i$ being $\alpha$ in
one case and $\gamma$ in the other), then the sentence should also be believed after that
sequence which differs from both
only in that step $i$ is a revision by the disjunction $\alpha \vee \gamma$.
Similarly, (Disj2) says that every sentence believed after an $(\alpha\vee\gamma)$-$\beta$-revision
should be believed after at least one of
$(\alpha$-$\beta)$ and $(\gamma$-$\beta)$.
Both conditions are reasonable properties to expect of revision operators.
\begin{proposition}
\label{disj}%
Restrained revision satisfies (Disj1) and (Disj2).
\end{proposition}

To prove this result we will make use of the properties of the counteracts relation given in
Proposition \ref{counteractsprops}, along with the following lemma.

\begin{lemma}
\label{extraprop}
If $\alpha \counteracts_{\E} \beta$ and
$(\alpha \vee \gamma) \not\counteracts_{\E} 
\beta$ then $B(\E * ((\alpha \vee \gamma) \wedge \beta)) = B(\E * (\gamma \wedge \beta))$
\end{lemma}

\begin{proof}
Suppose $\alpha \counteracts_{\E} \beta$ and
$(\alpha \vee \gamma) \not\counteracts_{\E} 
\beta$. We will first show that this implies $\neg\alpha \in B(\E * ((\alpha \vee \gamma) \wedge \beta))$. We will then be
able to conclude the required $B(\E * ((\alpha \vee \gamma) \wedge \beta)) = B(\E * (\gamma \wedge \beta))$ using RAGM. So
suppose on the contrary $\neg\alpha \not\in B(\E * ((\alpha \vee \gamma) \wedge \beta))$. Then there exists some 
$\alpha$-world $w \in
\min((\alpha \vee \gamma) \wedge \beta, \preceq_\E)$. Then also $w \in \min(\alpha \wedge \beta, \preceq_\E)$. Since 
$\alpha \counteracts_{\E} \beta$ we know from Proposition \ref{counter} 
there exist an $\alpha$-world $w_1$ and a $\beta$-world $w_2$ such that
$w_i \prec_\E w$ for $i = 1,2$. Clearly $w_1$ is also a $(\alpha \vee \gamma)$-world, so we infer $(\alpha \vee \gamma) \not\counteracts_{\E} 
\beta$ -- contradiction. Hence $\neg\alpha \in B(\E * ((\alpha \vee \gamma) \wedge \beta))$ as required.
\end{proof}

\begin{proof}[of Proposition \ref{disj}]
We prove both properties simultaneously by looking at two cases: \\
\underline{Case {\em (i)}: $(\alpha \vee \gamma) \counteracts_\E \beta$.}
In this case $B(\E * (\alpha \vee \gamma) * \beta) = B(\E * \beta)$ by property (T) in Proposition \ref{succinct}. Meanwhile 
we know from Proposition \ref{counteractsprops}{\em (ii)} that either 
$\alpha \counteracts_\E \beta$ or $\gamma \counteracts_\E \beta$, and so using (T) again we know
at least one of $B(\E * \alpha * \beta)$ and $B(\E * \gamma * \beta)$ must also be equal to $B(\E * \beta)$. Hence we
see both (Disj1) and (Disj2) hold in this case.

\noindent
\underline{Case {\em (ii)}: $(\alpha \vee \gamma) \not\counteracts_\E \beta$.} In this case (T) tells
us $B(\E * (\alpha \vee \gamma) * \beta) = B(\E * ((\alpha \vee \gamma) \wedge \beta)) =
B(\E * ((\alpha \wedge \beta) \vee (\gamma \wedge \beta))$. Meanwhile Proposition \ref{counteractsprops}{\em (i)} tells us at
least one of $\alpha \not\counteracts_\E \beta$ and $\gamma \not\counteracts_\E \beta$ holds. We now
consider two subcases according to which either both these hold, or only one holds. If both these hold then
$B(\E * \alpha * \beta) = B(\E * (\alpha \wedge \beta))$ and $B(\E * \gamma * \beta) = B(\E * (\gamma \wedge \beta))$, so
(Disj1) and (Disj2) reduce to
\begin{description}
\item[(Disj1$'$)]
$B(\E \ast (\alpha \wedge \beta)) \cap B(\E \ast (\gamma \wedge \beta)) \subseteq B(\E \ast ((\alpha \wedge \beta)
\vee (\gamma \wedge \beta)))$

\item[(Disj2$'$)]
$B(\E \ast ((\alpha \wedge \beta)
\vee (\gamma \wedge \beta))) \subseteq B(\E \ast (\alpha \wedge \beta)) \cup B(\E \ast (\gamma \wedge \beta))$
\end{description}
respectively. Now it is a consequence of RAGM that for {\em any} sentences $\theta, \phi$ we have both
$(1)$ $B(\E * \theta) \cap B(\E * \phi) \subseteq B(\E * (\theta \vee \phi))$ and $(2)$
$B(\E * (\theta \vee \phi)) \subseteq B(\E * \theta) \cup B(\E * \phi)$. Substituting $\alpha \wedge \beta$ for $\theta$
and $\gamma \wedge \beta$ for $\phi$ here gives us the required (Disj1$'$) (from $(1)$) and (Disj2$'$) (from $(2)$).

Now let's consider the subcase where $\alpha \counteracts_\E \beta$ and $\gamma \not\counteracts_\E \beta$. (A symmetric
argument will work for the other subcase $\alpha \not\counteracts_\E \beta$ and $\gamma \counteracts_\E \beta$.)
Then from $\gamma \not\counteracts_\E \beta$ we get $B(\E * \gamma * \beta) = B(\E * (\gamma \wedge \beta))$, while
from $\alpha \counteracts_\E \beta$ together with $(\alpha \vee \gamma) \not\counteracts_\E \beta$ we get
also $B(\E * (\alpha \vee \gamma) * \beta) = B(\E * (\gamma \wedge \beta))$ using Lemma \ref{extraprop}.
So in this case $B(\E * (\alpha \vee \gamma) * \beta) = B(\E * \gamma * \beta)$, from which both (Disj1) and (Disj2)
follow immediately.
\end{proof}

\noindent
We end this section by remarking that it can be shown that lexicographic revision also satisfies (Disj1) and (Disj2).

\section{Restrained Revision as a Composite Operator}
\label{Sec:Composite}

As we saw in Section \ref{Sec:admissible}, Boutilier's natural revision operator -- let us denote it
in this section by $\oplus$ -- is vulnerable to damaging counterexamples such as the red bird Example \ref{Ex:red-bird},
and fails to satisfy the very
reasonable postulate (P). Although a new input $\alpha$ is accepted in the very next
 epistemic state $\E \oplus \alpha$, $\oplus$ does not in any way provide for
the preservation of $\alpha$ after subsequent revisions. As Hans
Rott \citeyear[p.128]{Rott:2003a} describes it, ``[t]he most recent
input sentence is always embraced without reservation, the last but
one input sentence, however, is treated with utter disrespect''.
Thus, there seem to be convincing reasons to reject $\oplus$ as a
viable operator for performing iterated revision. However, the
literature on epistemic state change constantly reminds us that
keeping changes {\em minimal} should be a major concern, and when
judged from a purely {\em minimal change} viewpoint, it is clear
that $\oplus$ can't be beaten! How can we find our way out of this
apparent quandary? In this section we show that the use of $\oplus$
can be retained, {\em provided} its application is preceded by an
{\em intermediate} operation in which, rather than revising $\E$
{\em by} new input $\alpha$, essentially {\em $\alpha$ is revised by
$\E$}.


Given an epistemic state $\E$ and sentence $\alpha$, let us denote by $\E \tl \alpha$ the result of this intermediate
operation. $\E \tl \alpha$ is an epistemic state.
The idea is that when forming $\E \tl \alpha$, the information in $\E$ should be maintained. That is, the
total preorder $\preceq_{\E \tl \alpha}$ should satisfy
\begin{equation}
\label{dominate}
v \prec_{\E} w\
\textrm{implies}\
v \prec_{\E \tl \alpha} w.
\end{equation}
But rather than leaving behind $\alpha$ {\em entirely} in favour of
$\E$, as much of the informational content of $\alpha$ should be
preserved in $\E \tl \alpha$ as possible. This is formalised by
saying that for any $v \in [\alpha], w \in [\neg\alpha]$, we should
take $v \prec_{\E \tl \alpha} w$ as long as this does not conflict
with (\ref{dominate}) above. It is this second requirement which
will guarantee $\alpha$ enough of a ``presence'' in the revised
epistemic state $\E \ast \alpha$ to help it survive subsequent
revisions and allow (P) to be captured. Taken together, the above
two requirements are enough to specify $\preceq_{\E \tl \alpha}$
uniquely:
\begin{equation}
\label{order}
v\preceq _{\E \tl \alpha}w\textrm{ iff }\left\{
\begin{array}{l}
v\prec _{\E}w\textrm{, or}\\
v\preceq _{\E}w\textrm{ and }(v\in [\alpha]\textrm{ or }w\in [\neg\alpha])\textrm{.} 

\end{array}
\right.
\end{equation}
Thus, $\preceq_{\E \tl \alpha}$ is just the lexicographic refinement
of $\preceq_{\E}$ by the ``two-level'' total preorder $\preceq_\alpha$
defined by $v \preceq_\alpha w$ iff $v \in [\alpha]$ or $w \in
[\neg\alpha]$. This ``backwards revision'' operator is not new. It
has been studied by Papini \citeyear{Papini:98b}. It can also
be viewed as just a ``backwards'' version of Nayak's lexicographic
revision operator. 
We do not necessarily have $\alpha \in B(\E \tl
\alpha)$ (this will hold only if $\neg\alpha \not\in B(\E)$), and so
$\tl$ does not satisfy RAGM.

Given $\tl$, we can define the composite revision operator $\ast_\tl$ by setting
\begin{equation}
\label{itlevi}
\E \ast_\tl \alpha = (\E \tl \alpha) \oplus \alpha
\end{equation}
This is reminiscent of the Levi Identity \cite{gardenfors:88a}, used
in AGM theory as a recipe for reducing the operation of revision on
{\em knowledge bases} to a composite operation consisting of
contraction plus expansion. In (\ref{itlevi}), $\oplus$ is playing
the role of expansion. The operator $\ast_\tl$ {\em does} satisfy
RAGM. In fact, as can easily be seen by comparing (\ref{order})
above with condition (RR) at the start of Section 4, $\ast_\tl$
coincides with restrained revision.
\begin{proposition}
Let $\ast_{\R}$ denote the restrained revision operator. Then $\ast_{\R} = \ast_\tl$.
\end{proposition}
Thus we have proved that restrained revision can be viewed as a {\em combination} of two existing operators.

\section{How to Choose a Revision Operator}
\label{Sec:Choice}%
The contribution of this paper so far can be summarised as follows. We have argued for the replacement
of the Darwiche-Pearl framework by the class of admissible revision operators, arguing that the
former needs to be strengthened. In doing so we have eliminated natural revision,  but retained lexicographic revision and the $\bullet$ operator of Darwiche and Pearl as admissible operators. We have also introduced a new admissible revision operator, restrained revision, and argued for its plausibility. But this is not an argument that restrained revision is somehow unique, or more
preferred than other revision operators. The contention is merely that, for epistemic state revision, the Darwiche-Pearl framework is too weak and should be replaced by admissible revision. And restrained 
revision, being an admissible revision operator, is therefore only one of many revision operators deemed to be rational. The
question of which admissible revision operator to use in a particular situation
is one which depends on a number of issues, such as context, the strength with which certain beliefs are held, the source of the information, and so on. This point has essentially also been
made by Friedman and Halpern \citeyear{Friedman-ea:99a}. For example, Example 
\ref{Ex:red-bird-2} formed part of  an argument for the use of restrained revision, and against the use of lexicographic revision. In effect we used the example to argue that 
$\mathtt{red}$ ought to be in $\mathit{B}(\mathtt{\E\ast red \rightarrow bird\ast\lnot bird})$, where $B(\E)= Cn(\tt{red})$. But if we change the context slightly, it becomes an example in \emph{favour} of the use of lexicographic revision, and \emph{against} restrained revision. 
 \begin{example}
\label{Ex:red-bird-3} We observe a creature which seems to be red,
but we are too far away to determine whether it is a  bird or a land
animal. So we adopt the knowledge base $B(\E)= Cn(\mathtt{red})$. Next to us
is an expert on birds who remarks that, if the creature is indeed red, it must be
a bird. So we adopt the belief
$\mathtt{red\rightarrow bird}$. Then we get information from someone standing closer to the creature
that it is not a bird. Given this context, that is, the reliability of the expert combined with the statement that the creature 
initially seemed to be red, it is reasonable to adopt the lexicographic approach of ``more recent is best'' and conclude that the bird is not red. Formally,  $\mathtt{\lnot red}\in \mathit{B}(\mathtt{\E\ast red \rightarrow bird\ast\lnot bird})$, where $B(\E) = Cn(\tt{red})$. 
\end{example}
For a case where the source of the information dramatically affects the outcome, consider the following
example.
\begin{example}
\label{Ex:repeat}%
Consider the sequence of inputs where $p$ is followed by a finite number, say $n$,  of instances of the pair $p\rightarrow q, \lnot q$. (To make this more concrete, the reader might wish to substitute $p$
with $\mathtt{red}$ and $q$ with $\mathtt{bird}$.)
Since $p$ is not in direct conflict with any sentences in the sequence 
succeeding it, any revision operator satisfying the property O (and this includes restrained revision) will require that $p$ be contained in the knowledge base obtained from this revision sequence. Now, if each pair $p\rightarrow q$ and $\lnot q$
is obtained from a \emph{different} source, such a conclusion is clearly unreasonable. After all, such a sequence amounts to being told that $p$ is the case, followed by  $n$ different sources essentially telling you that $\lnot p$ is the case. 
On the other hand, if the pairs $p\rightarrow q$ and $\lnot q$ all come from the same source, the case is not so clear cut anymore. In fact, in this case one would expect the result to be the same as that 
obtained from the sequence $p, p\rightarrow q, \lnot q$, a sequence with the same formal structure as that employed in Example \ref{Ex:red-bird-2}, where restrained revision was seen to be a reasonable 
approach.
\end{example}

\noindent
Another example in which restrained revision fares less well is the following:\footnote{We are grateful to one
of the anonymous referees for suggesting this example.}
\begin{example}
Suppose we are teaching a class of students consisting of $n$ boys and $m$ girls, and
suppose the class takes part in a mathematics competition. For each $i = 1, \ldots, n$ and
$j = 1, \ldots, m$ let
the propositional variables $p_i$ and $q_j$ stand for ``boy $i$ won the competition''
and ``girl $j$ won the competition'' respectively, and suppose
initially we believe one of the boys won the competition, i.e., $B(\E) =
Cn(\phi \wedge \Sigma)$ where $\phi = \bigvee_i p_i$ and $\Sigma$ is just some sentence
expressing the uniqueness of the competition winner. Now suppose we interview each of the boys one
after the other, and each of them tells us that either one of the girls or he himself
won, i.e., we obtain the sequence of inputs $(\neg\phi \vee p_i)_i$. Suppose we are willing
to accept the boys' testimony. Using a revision operator which satisfies O will lead us to
believe boy $n$ won the competition, which seems implausible. Lexicographic revision gives the
desired result that one of the girls won the competition.
\end{example}

From these examples it is clear that an agent need not, and in most cases, \emph{ought} not to stick to the same revision operator every time that it has to perform a revision. This means that  the agent will  
keep on switching from one revision operator to another during the process of iterated revision.  Of course, this leads to the question of how to choose among the available (admissible) revision operators at any particular point.
A comprehensive answer to this question is beyond the scope of this paper, but we do provide some clues on how to address the problem. In brief, we contend that epistemic states have to be enriched, with a more
detailed specification of their internal structure. Looking back at the history of belief revision, we can see
that this is exactly how the field has progressed. In the initial papers, such as those on 
AGM revision, an epistemic state was taken 
to contain nothing more than a knowledge base. So, for example, basic AGM revision as characterised 
by the first six AGM postulates imposes no structure on epistemic states at all. 
We shall refer to these as \emph{simple} epistemic states.  With full 
AGM belief revision as characterised by all eight AGM postulates, the view is still one of the 
revision of knowledge bases, but now every revision 
operator for knowledge base $B$ is uniquely associated with a $B$-faithful total preorder; 
i.e., a total preorder on valuations with the models of $B$ as its minimal elements.
From here it is a small step to \emph{define} epistemic states to include such an ordering, i.e. 
to include the total preorder $\preceq_\E$ associated with an epistemic state $\E$ as part of the definition of $\E$.
We shall refer to these as 
\emph{complex} epistemic states.

This leads to two different views of the same revision process. If we view revision as an
operator on simple epistemic states we have many different revision operators; one corresponding to each of the $B$-faithful total preorders, but with no way to distinguish between them when having to choose a revision operator. Viewed as such, iterated revision is a process in which a (possibly) 
different revision operator is employed at every revision step. This is the principal view adopted 
by Nayak et al. \citeyear{Nayak-ea:2003a}.  However, if we view revision as an operator on complex 
epistemic states, every epistemic state  contains enough 
information to determine \emph{uniquely} the knowledge base, but not the faithful total preorder, resulting from the revision. In other words, we now have enough information encoded in an epistemic state to uniquely determine the knowledge base resulting from a revision, but we lack the information
to uniquely determine the full epistemic state. The Darwiche-Pearl framework, and now also admissible revision, place some constraints on the resulting epistemic state, but do not impose any 
additional structure on the complex epistemic state. In our view admissible revision for 
complex epistemic states is analogous to basic AGM revision for simple epistemic states. The next step would thus be to impose additional structure on complex epistemic states. This could possibly involve the addition of a second ordering
on valuations as was done, for example, by Booth et al. \citeyear{Booth-ea:2004a}.  In the case of simple epistemic states the effect of adding the two supplementary postulates is to constrain basic revision to the extent that each revision operator can be uniquely associated with a $B$-faithful total preorder. In that sense, the addition of the supplementary postulates allowed for the imposition of 
additional structure on simple epistemic states. Recall that one way of interpreting the two supplementary postulates is 
that they explain the interaction between revision by two sentences and revision by their conjunction, something the basic postulates do not address. 

So, as we have seen, the addition of the supplementary postulates leads to the definition of revision as operators on 
complex epistemic states. We conjecture that giving additional structure to complex epistemic states might involve the provision of postulates analogous to the two AGM supplementary postulates. In particular, we conjecture that such 
postulates might be such that they explain the interaction between two sentences and their conjunction, or disjunction, for
\emph{iterated revision}. Observe that none of the Darwiche-Pearl postulates, or the additional
postulates for admissible revision for that matter, address this issue. In fact the only
postulates to have been suggested (of which we are aware) which so far {\em do} address it are (Disj1) and (Disj2). We speculate that the
appropriate set of supplementary postulates for iterated revision (which may or may not include
the two just mentioned) will lead to the definition 
of extra structure on complex epistemic states, which can then be incorporated into an 
enriched version of complex epistemic states, with revision then being seen as 
operators on these enriched entities. Let us refer to them as \emph{enriched}
epistemic states. Enriched epistemic states will enable us to determine uniquely 
the complex epistemic state resulting from a revision, thereby solving the question
we started off with; that of determining which revision on complex epistemic states to 
use at every particular point during a process of iterated revision.   Below we shall briefly discuss a possible way of enriching complex epistemic states. But note also that a recent proposal for
doing so is that of Booth et al.  \citeyear{Booth-ea:2006a}. It is instructive to observe that 
 (Disj1)  and (Disj2) both hold in their framework. 

The proposed outline above  is not without its pitfalls. The most obvious problem with such 
an approach is that it leaves us with a meta-version of the dilemma that we started off with.
Using enriched epistemic states we are now able to uniquely determine the complex epistemic 
states resulting from a revision, but not the resulting \emph{enriched} epistemic state.  We can 
lessen the problem by constraining the permissible resulting enriched epistemic states in the same way that admissible revision constrains the permissible complex epistemic states, but chances are that
this whittling down will not produce a single permissible enriched epistemic state.
And, of course, this is bound to occur over and over again. That is, whenever we solve the problem of uniquely  determining an epistemic state with a certain structure by a process of further enrichment, we will be saddled with the question of how to uniquely determine the further enriched epistemic state resulting from a revision. Our conjecture is that at some level a point will be reached where  
constraining further enriched epistemic states, \emph{{\'a} la} admissible revision, will eventually lead to a
unique further enriched epistemic state associated with every revision.  Only further research 
will determine whether our conjecture holds water. 

In conclusion, we have shown that the Darwiche-Pearl arguments lead to the
acceptance of the admissible revision operators as a class worthy of
study. The restrained revision operator, in particular, exhibits
quite desirable properties. Besides taking the place of natural
revision as the operator adhering most closely to the principle of
minimal change, its satisfaction of the properties (O), (Q) and (U) shows
that it does not unnecessarily remove previously obtained
information.

For future work we would also like to explore more thoroughly the class
of admissible revision operators. In this paper we saw that
restrained revision and lexicographic revision lie at opposite ends
of the spectrum of admissible operators. They represent respectively
the most conservative and the least conservative admissible
operators in the sense that they effect the most changes and the least changes, 
respectively, in the relative ordering of valuations permitted by admissible revision.
A natural question is whether there exists an
axiomatisable class of admissible operators which represents the
``middle ground''. One clue for finding such a class can be found in
the counteracts relation $\counteracts_{\E}$ which can be derived from
an epistemic state $\E$. As we said, this relation depends only on
the preorder $\preceq_{\E}$ associated to $\E$. In fact, given {\em
any} total preorder $\preceq$ over $V$ we can define the relation
$\counteracts_\preceq$ by
\[
\alpha \counteracts_\preceq \beta\
\textrm{iff}\
\min(\alpha, \preceq) \subseteq [\lnot\beta]\
\textrm{and}\
\min(\beta, \preceq) \subseteq [\lnot\alpha].\
\]
Then clearly $\counteracts_{\E} = \counteracts_{\preceq_{\E}}$.
Furthermore if $\preceq$ is the full relation $V \times V$ then
$\counteracts_\preceq$ reduces to logical inconsistency. A
counteracts relation {\em stronger} than $\counteracts_{\E}$, but
still {\em weaker} than logical inconsistency can be found by
setting $\counteracts = \counteracts_{\preceq'}$, where $\preceq'$
lies somewhere {\em in between} $\preceq_{\E}$ and $V \times V$. Hence
one avenue worth exploring might be to assume that from each
epistemic state $\E$ we can extract not one but {\em two} preorders
$\preceq_{\E}$ and $\preceq'_{\E}$ such that $\preceq_{\E} \subseteq
\preceq'_{\E}$. Then, instead of only requiring $\alpha
\counteracts_{\E} \beta$ to deduce $\lnot\alpha \in B(\E \ast \alpha
\ast \beta)$, as is done with restrained revision (the postulate
(D)), we could require the stronger condition $\alpha
\counteracts_{\preceq'_{\E}} \beta$ for this to hold. We are currently
experimenting with strategies for using the second preorder to {\em
guide} the manipulation of $\preceq_{\E}$ to enable this property to
be satisfied. The use of a second preorder can be seen 
as a way of enriching the epistemic state, and might thus contribute
to the solution of the  choice of revision operators discussed in 
Section \ref{Sec:Choice}.

Some more future work relates also to the two extreme cases revision,  but looked at from a different angle. As mentioned earlier, lexicographic revision is a formalisation of the ``most recent 
is best'' approach to revision taken to its logical extreme. This approach is exemplified by the 
($\E\ast$2) postulate, also known as Success, which requires a revision to be successful, in the 
sense that the epistemic input provided always has to be contained in the resulting knowledge base.  
Given that ($\E\ast$2) is one of the postulates for admissible revision, this requirement carries 
over even to restrained revision, which is on the opposite end of the spectrum for admissible
revision. But this means that the admissible revision operator which differs the most from 
lexicographic revision still adheres to the dictum of ``most recent is best'', which raises 
the question of why the most recent input is given such prominence. The relaxation of 
this requirement would imply giving up  ($\E\ast$2) and venturing into the area known 
as \emph{non-prioritised} revision \cite{Booth:2001a,Chopra-ea:2003a,Hanssonnpr}.   We speculate that an appropriate relaxation
of admissible revision, with ($\E\ast$2)  removed as a requirement, will lead to a class of 
(non-prioritised) revision operators strictly containing admissible revision, and with lexicographic revision still at one end of the spectrum, but with the other end of the spectrum occupied by the
operator studied by Papini \citeyear{Papini:98b} which was used as a sub-operation of restrained revision
in Section \ref{Sec:Composite}. This operator is
formalised by the extreme version of ``most recent is worst''; in other words, 
``the older the better''.
 
\section*{Acknowledgements}
Much of the first author's work was done during stints as a
researcher at Wollongong University and Macquarie University, Sydney.
He wishes to thank Aditya Ghose and Abhaya Nayak both for making it
possible to enjoy the great working environments there, and also for
some interesting comments on this work.
Thanks are also due to Samir Chopra who contributed to a preliminary version of the paper,
Adnan Darwiche for clearing up some misconceptions on the definition of epistemic states, 
and three anonymous referees for their valuable and insightful comments. 
National ICT Australia is funded by the Australia Government's Department of Communications,
 Information and Technology and the Arts and the Australian Research Council through Backing
 Australia's Ability and the ICT Centre of Excellence program.
 It is supported by its members the Australian National University, University of NSW, ACT Government,
 NSW Government and affiliate partner University  of Sydney.
\bibliographystyle{theapa}
\bibliography{jair-rev5}

\begin{thebibliography}{}

\bibitem[\protect\BCAY{Alchourr{\'o}n, G{\"a}rdenfors,\ \BBA\
  Makinson}{Alchourr{\'o}n et~al.}{1985}]{Alchourron-ea:85a}
Alchourr{\'o}n, C.~E., G{\"a}rdenfors, P., \BBA\ Makinson, D. \BBOP1985\BBCP.
\newblock \BBOQ On the logic of theory change: Partial meet functions for
  contraction and revision\BBCQ\
\newblock {\Bem Journal of Symbolic Logic}, {\Bem 50}, 510--530.

\bibitem[\protect\BCAY{Areces\ \BBA\ Becher}{Areces\ \BBA\
  Becher}{2001}]{Areces-ea:2001a}
Areces, C.\BBACOMMA\  \BBA\ Becher, V. \BBOP2001\BBCP.
\newblock \BBOQ Iterable {AGM} functions\BBCQ\
\newblock In {\Bem Frontiers in belief revision}, \BPGS\ 261--277. Kluwer,
  Dordrecht.

\bibitem[\protect\BCAY{Booth}{Booth}{2001}]{Booth:2001a}
Booth, R. \BBOP2001\BBCP.
\newblock \BBOQ A negotiation-style framework for non-prioritised
  revision\BBCQ\
\newblock In van Benthem, J.\BED, {\Bem Theoretical Aspects of Rationality and
  Knowledge: Proceedings of the Eighth Conference (TARK 2001)}, \BPGS\
  137--150, San Francisco, California. Morgan Kaufmann.

\bibitem[\protect\BCAY{Booth}{Booth}{2005}]{Booth:2005a}
Booth, R. \BBOP2005\BBCP.
\newblock \BBOQ On the logic of iterated non-prioritised revision\BBCQ\
\newblock In {\Bem Conditionals, Information and Inference -- Selected papers
  from the Workshop on Conditionals, Information and Inference, 2002},
  \lowercase{\BVOL}\ 3301 of {\Bem LNAI}, \BPGS\ 86--107. Springer-Verlag,
  Berlin.

\bibitem[\protect\BCAY{Booth, Chopra, Ghose,\ \BBA\ Meyer}{Booth
  et~al.}{2004}]{Booth-ea:2004a}
Booth, R., Chopra, S., Ghose, A., \BBA\ Meyer, T. \BBOP2004\BBCP.
\newblock \BBOQ A unifying semantics for belief change\BBCQ\
\newblock In Mantaras, R. L.~D.\BBACOMMA\  \BBA\ Saitta, L.\BEDS, {\Bem
  Sixteenth European Conference on Artificial Intelligence: ECAI2004}, \BPGS\
  793--797. IOS Press.

\bibitem[\protect\BCAY{Booth, Meyer,\ \BBA\ Wong}{Booth
  et~al.}{2006}]{Booth-ea:2006a}
Booth, R., Meyer, T., \BBA\ Wong, K.-S. \BBOP2006\BBCP.
\newblock \BBOQ A bad day surfing is better than a good day working: {H}ow to
  revise a total preorder\BBCQ\
\newblock In {\Bem Proceedings of KR2006, Tenth International Conference on the
  Principles of Knowledge Representation and Reasoning}.

\bibitem[\protect\BCAY{Boutilier}{Boutilier}{1993}]{Boutilier:93a}
Boutilier, C. \BBOP1993\BBCP.
\newblock \BBOQ Revision sequences and nested conditionals\BBCQ\
\newblock In Bajcsy, R.\BED, {\Bem IJCAI-93. Proceedings of the 13th
  International Joint Conference on Artificial Intelligence held in Chambery,
  France, August 28 to September 3, 1993}, \lowercase{\BVOL}~1, \BPGS\
  519--525, San Mateo, CA. Morgan Kaufmann.

\bibitem[\protect\BCAY{Boutilier}{Boutilier}{1996}]{Boutilier:96a}
Boutilier, C. \BBOP1996\BBCP.
\newblock \BBOQ Iterated revision and minimal changes of conditional
  beliefs\BBCQ\
\newblock {\Bem Journal of Philosophical Logic}, {\Bem 25\/}(3), 263--305.

\bibitem[\protect\BCAY{Chopra, Ghose,\ \BBA\ Meyer}{Chopra
  et~al.}{2003}]{Chopra-ea:2003a}
Chopra, S., Ghose, A., \BBA\ Meyer, T. \BBOP2003\BBCP.
\newblock \BBOQ Non-prioritized ranked belief change\BBCQ\
\newblock {\Bem Journal of Philosophical Logic}, {\Bem 32\/}(3), 417--443.

\bibitem[\protect\BCAY{Dalal}{Dalal}{1988}]{Dalal:88a}
Dalal, M. \BBOP1988\BBCP.
\newblock \BBOQ Investigations into a theory of knowledge base revision\BBCQ\
\newblock In {\Bem Proceedings of the 7th National Conference of the American
  Association for Artificial Intelligence, Saint Paul, Minnesota}, \BPGS\
  475--479.

\bibitem[\protect\BCAY{Darwiche\ \BBA\ Pearl}{Darwiche\ \BBA\
  Pearl}{1997}]{Darwiche-ea:97a}
Darwiche, A.\BBACOMMA\  \BBA\ Pearl, J. \BBOP1997\BBCP.
\newblock \BBOQ On the logic of iterated belief revision\BBCQ\
\newblock {\Bem Artificial Intelligence}, {\Bem 89}, 1--29.

\bibitem[\protect\BCAY{Freund\ \BBA\ Lehmann}{Freund\ \BBA\
  Lehmann}{1994}]{Freund-ea:94a}
Freund, M.\BBACOMMA\  \BBA\ Lehmann, D. \BBOP1994\BBCP.
\newblock \BBOQ Belief revision and rational inference\BBCQ\
\newblock \BTR\ TR 94-16, The Leibniz Centre for Research in Computer Science,
  Institute of Computer Science, Hebrew University of Jerusalem.

\bibitem[\protect\BCAY{Friedman\ \BBA\ Halpern}{Friedman\ \BBA\
  Halpern}{1999}]{Friedman-ea:99a}
Friedman, N.\BBACOMMA\  \BBA\ Halpern, J.~Y. \BBOP1999\BBCP.
\newblock \BBOQ Belief revision: A critique\BBCQ\
\newblock {\Bem Journal of Logic, Language and Information}, {\Bem 8},
  401--420.

\bibitem[\protect\BCAY{G{\"a}rdenfors}{G{\"a}rdenfors}{1988}]{gardenfors:88a}
G{\"a}rdenfors, P. \BBOP1988\BBCP.
\newblock {\Bem Knowledge in Flux : Modeling the Dynamics of Epistemic States}.
\newblock The MIT Press, Cambridge, Massachusetts.

\bibitem[\protect\BCAY{Glaister}{Glaister}{1998}]{Glaister:98a}
Glaister, S.~M. \BBOP1998\BBCP.
\newblock \BBOQ Symmetry and belief revision\BBCQ\
\newblock {\Bem Erkenntnis}, {\Bem 49}, 21--56.

\bibitem[\protect\BCAY{Grove}{Grove}{1988}]{Grove:88a}
Grove, A. \BBOP1988\BBCP.
\newblock \BBOQ Two modellings for theory change\BBCQ\
\newblock {\Bem Journal of Philosophical Logic}, {\Bem 17}, 157--170.

\bibitem[\protect\BCAY{Hansson}{Hansson}{1999}]{Hanssonnpr}
Hansson, S.~O. \BBOP1999\BBCP.
\newblock \BBOQ A survey of non-prioritized belief revision\BBCQ\
\newblock {\Bem Erkenntnis}, {\Bem 50}, 413--427.

\bibitem[\protect\BCAY{Jin\ \BBA\ Thielscher}{Jin\ \BBA\
  Thielscher}{2005}]{Jin-ea:2005a}
Jin, Y.\BBACOMMA\  \BBA\ Thielscher, M. \BBOP2005\BBCP.
\newblock \BBOQ Iterated belief revision, revised\BBCQ\
\newblock In {\Bem Proceedings of the Nineteenth International Joint Conference
  on Artificial Intelligence (IJCAI 05)}, \BPGS\ 478--483.

\bibitem[\protect\BCAY{Katsuno\ \BBA\ Mendelzon}{Katsuno\ \BBA\
  Mendelzon}{1991}]{Katsuno-ea:91a}
Katsuno, H.\BBACOMMA\  \BBA\ Mendelzon, A.~O. \BBOP1991\BBCP.
\newblock \BBOQ Propositional knowledge base revision and minimal change\BBCQ\
\newblock {\Bem Artificial Intelligence}, {\Bem 52}, 263--294.

\bibitem[\protect\BCAY{Konieczny\ \BBA\ {Pino P\'erez}}{Konieczny\ \BBA\ {Pino
  P\'erez}}{2000}]{KP}
Konieczny, S.\BBACOMMA\  \BBA\ {Pino P\'erez}, R. \BBOP2000\BBCP.
\newblock \BBOQ A framework for iterated revision\BBCQ\
\newblock {\Bem Journal of Applied Non-Classical Logics}, {\Bem 10(3-4)},
  339--367.

\bibitem[\protect\BCAY{Lehmann}{Lehmann}{1995}]{Lehmann}
Lehmann, D. \BBOP1995\BBCP.
\newblock \BBOQ Belief revision, revised\BBCQ\
\newblock In {\Bem Proceedings of the Fourteenth International Joint Conference
  on Artificial Intelligence (IJCAI'95)}, \BPGS\ 1534--1540.

\bibitem[\protect\BCAY{Lehmann, Magidor,\ \BBA\ Schlechta}{Lehmann
  et~al.}{2001}]{Lehmann-ea:2001a}
Lehmann, D., Magidor, M., \BBA\ Schlechta, K. \BBOP2001\BBCP.
\newblock \BBOQ Distance semantics for belief revision\BBCQ\
\newblock {\Bem Journal of Symbolic Logic}, {\Bem 66}, 295--317.

\bibitem[\protect\BCAY{Lewis}{Lewis}{1973}]{Lewis:73a}
Lewis, D.~K. \BBOP1973\BBCP.
\newblock \BBOQ Counterfactuals\BBCQ\
\newblock {\Bem Journal of Philosophy}, {\Bem 70}, 556--567.

\bibitem[\protect\BCAY{Nayak}{Nayak}{1993}]{Nayak:93a}
Nayak, A.~C. \BBOP1993\BBCP.
\newblock {\Bem Studies in {B}elief {C}hange}.
\newblock Ph.D.\ thesis, University of Rochester.

\bibitem[\protect\BCAY{Nayak}{Nayak}{1994}]{Nayak:94b}
Nayak, A.~C. \BBOP1994\BBCP.
\newblock \BBOQ Iterated belief change based on epistemic entrenchment\BBCQ\
\newblock {\Bem Erkenntnis}, {\Bem 41}, 353--390.

\bibitem[\protect\BCAY{Nayak, Foo, Pagnucco,\ \BBA\ Sattar}{Nayak
  et~al.}{1996}]{Nayak-ea:96b}
Nayak, A.~C., Foo, N.~Y., Pagnucco, M., \BBA\ Sattar, A. \BBOP1996\BBCP.
\newblock \BBOQ Changing {C}onditional {B}elief {U}nconditionally\BBCQ\
\newblock In Shoham, Y.\BED, {\Bem Theoretical Aspects of Rationality and
  Knowledge: Proceedings of the Sixth Conference (TARK 1996)}, \BPGS\ 119--136,
  San Francisco, California. Morgan Kaufmann.

\bibitem[\protect\BCAY{Nayak, Pagnucco,\ \BBA\ Peppas}{Nayak
  et~al.}{2003}]{Nayak-ea:2003a}
Nayak, A.~C., Pagnucco, M., \BBA\ Peppas, P. \BBOP2003\BBCP.
\newblock \BBOQ Dynamic belief change operators\BBCQ\
\newblock {\Bem Artificial Intelligence}, {\Bem 146}, 193--228.

\bibitem[\protect\BCAY{Papini}{Papini}{2001}]{Papini:98b}
Papini, O. \BBOP2001\BBCP.
\newblock \BBOQ Iterated revision operations stemming from the history of an
  agent's observations\BBCQ\
\newblock In {\Bem Frontiers in belief revision}, \BPGS\ 281--303. Kluwer,
  Dordrecht.

\bibitem[\protect\BCAY{Rott}{Rott}{2000}]{dogmas}
Rott, H. \BBOP2000\BBCP.
\newblock \BBOQ Two dogmas of belief revision\BBCQ\
\newblock {\Bem Journal of Philosophy}, {\Bem 97}, 503--522.

\bibitem[\protect\BCAY{Rott}{Rott}{2003}]{Rott:2003a}
Rott, H. \BBOP2003\BBCP.
\newblock \BBOQ Coherence and conservatism in the dynamics of belief {II}:
  {I}terated belief change without dispositional coherence\BBCQ\
\newblock {\Bem Journal of Logic and Computation}, {\Bem 13\/}(1), 111--145.

\bibitem[\protect\BCAY{Schlechta, Lehmann,\ \BBA\ Magidor}{Schlechta
  et~al.}{1996}]{Schlechta-ea:96a}
Schlechta, K., Lehmann, D., \BBA\ Magidor, M. \BBOP1996\BBCP.
\newblock \BBOQ Distance semantics for belief revision\BBCQ\
\newblock In Shoham, Y.\BED, {\Bem Proceedings of the Sixth Conference on
  Theoretical Aspects of Rationality and Knowledge}, \BPGS\ 137--145. Morgan
  Kaufmann.

\bibitem[\protect\BCAY{Segerberg}{Segerberg}{1998}]{Segerberg:98a}
Segerberg, K. \BBOP1998\BBCP.
\newblock \BBOQ Irrevocable belief revision in dynamic doxastic logic\BBCQ\
\newblock {\Bem Notre Dame Journal of Formal Logic}, {\Bem 39}, 287--306.

\bibitem[\protect\BCAY{Spohn}{Spohn}{1988}]{Spohn:88a}
Spohn, W. \BBOP1988\BBCP.
\newblock \BBOQ Ordinal conditional functions: A dynamic theory of epistemic
  states\BBCQ\
\newblock In Harper, W.~L.\BBACOMMA\  \BBA\ Skyrms, B.\BEDS, {\Bem Causation in
  Decision: Belief, Change and Statistics: Proceedings of the Irvine Conference
  on Probability and Causation: Volume II}, \lowercase{\BVOL}~42 of {\Bem The
  University of Western Ontario Series in Philosophy of Science}, \BPGS\
  105--134, Dordrecht. Kluwer Academic Publishers.

\end{thebibliography}
\end{document}